\documentclass{article}

\usepackage[preprint]{neurips_2026}

\usepackage[utf8]{inputenc}
\usepackage[T1]{fontenc}
\usepackage{amsmath,amssymb,amsfonts}
\usepackage{booktabs}
\usepackage{graphicx}
\PassOptionsToPackage{hyphens}{url}
\usepackage[breaklinks=true]{hyperref}
\usepackage{url}
\usepackage{multirow}
\usepackage{xcolor}
\usepackage{subcaption}
\usepackage{amsthm}

\newtheorem{proposition}{Proposition}
\newtheorem{theorem}{Theorem}
\newtheorem{corollary}{Corollary}
\theoremstyle{remark}
\newtheorem{remark}{Remark}

\usepackage{titlesec}
\titlespacing*{\section}{0pt}{6pt plus 1pt minus 1pt}{3pt plus 1pt minus 1pt}
\titlespacing*{\subsection}{0pt}{4pt plus 1pt minus 1pt}{2pt plus 1pt minus 1pt}
\titlespacing*{\paragraph}{0pt}{3pt plus 1pt minus 1pt}{1em}

\graphicspath{{figures/}}

\title{Don't Learn the Shape:\\Forecasting Periodic Time Series by Rank-1 Decomposition\thanks{Code: \url{https://github.com/TakatoHonda/FLAIR}.}}

\author{
  Takato Honda \\
  Mellon Inc. \\
  \texttt{takato@melloninc.jp}
}

\begin{document}

\maketitle

\begin{abstract}
How few parameters do we really need to forecast a periodic time series? An hourly electricity series, reshaped as a $24$-row matrix with one column per day, is approximately rank-1: a daily shape modulated by a daily level (median centered rank-1 energy $0.82$ on GIFT-Eval).

Should we learn the shape? Smoothing, shrinkage, and low-rank fits all seem like obvious upgrades over the simple average of the last $K{=}2$ cycles. On all $97$ GIFT-Eval configurations, we tested $8$ such alternatives (e.g., Fourier, EWMA, James-Stein, rank-$r$ SVD): none significantly beats the frozen baseline under Holm correction; two are significantly worse.

The resulting method, FLAIR, is (a) \textbf{Effective}: matches PatchTST on aggregate GIFT-Eval (relMASE $0.838$ vs $0.849$); (b) \textbf{Compact}: $28$ scalars for hourly, $57$ for weekly; (c) \textbf{Fast}: $22$ minutes on one CPU core of a MacBook Pro; (d) \textbf{Closed-form \& Hands-Off}: one SVD per period candidate, GCV-averaged Ridge, no GPU, no pre-training, no per-task tuning. In the high-rank-1, many-cycle regime, extra flexibility is estimation noise.
\end{abstract}

\section{Introduction}
\label{sec:intro}

Once a periodic series reshapes into a rank-1 matrix, how much of an $8.3$B-parameter foundation model~\citep{timer-s1, ansari2024chronos, woo2024moirai, das2024timesfm} do we still need? Our answer: matching PatchTST on GIFT-Eval needs $28$ scalars per series, one CPU core, no GPU, and no pre-training. The path is one observation about periodic data, one sharp negative result, and one closed-form recipe.

\begin{figure}[t]
    \centering
    \includegraphics[width=\textwidth]{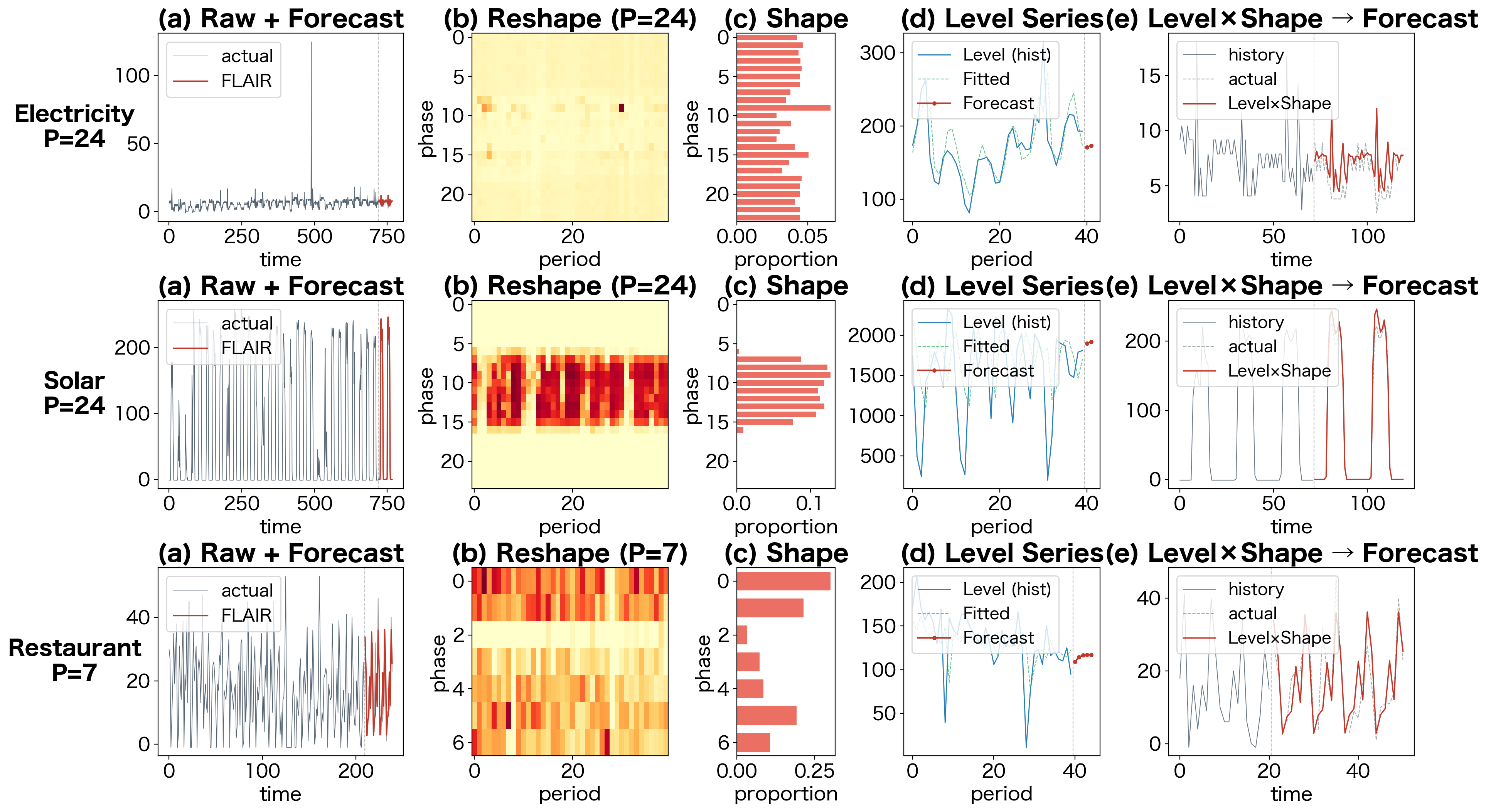}
    \caption{\textbf{Don't learn the shape; learn the level.} FLAIR factorizes a periodic series $y$ as $\hat y_h = \hat L_{\lceil h/P\rceil}\cdot S_{h\,\mathrm{mod}\,P}$, with the Shape $S$ frozen as the last-$K{=}2$ within-period proportion average and the Level $\hat L$ forecast by GCV-averaged Ridge. Hourly $P{=}24$ fits in $28$ scalars per series; aggregate relMASE on $97$ GIFT-Eval configurations is $0.838$, ties PatchTST. Each row: \textbf{(a)} raw series with $48$-step forecast (red); \textbf{(b)} reshape $y \to \mathbf{M} \in \mathbb{R}^{P\times n_c}$ (rank-1 structure as horizontal bands); \textbf{(c)} Shape $\mathbf{S}$; \textbf{(d)} Level $\mathbf{L}$ = column sums, forecast by Ridge; \textbf{(e)} reconstructed forecast.}
    \label{fig:rank1}
\end{figure}

\paragraph{Observation: periodic series are rank-1.}
A periodic series of period $P$, reshaped into a $P \times n_c$ matrix (one column per cycle; $n_c$ = number of complete cycles), is approximately rank-1: on GIFT-Eval, the per-configuration median centered rank-1 energy is $r_1 \equiv \sigma_1^2 / \sum_i \sigma_i^2 = 0.82$ (Section~\ref{sec:observation}; we use the row-centered $r_1$ that subtracts each phase's mean to avoid double-counting any constant offset). The forecast factorizes
\begin{equation}
    \hat{y}(\text{phase}, \text{period}) \;=\; \hat{L}(\text{period}) \;\times\; S(\text{phase}),
    \label{eq:core}
\end{equation}
and only the cycle-to-cycle Level $\hat L$ is worth learning; the within-cycle Shape $S$ is approximately invariant.

\paragraph{Negative result: don't learn the shape.}
Smoothing, spectral fits, shrinkage, and low-rank SVD all seem like obvious upgrades over the simple average of the last $K{=}2$ cycles: we test $8$ variants (EWMA; Fourier $J{=}1$; Savitzky-Golay; James-Stein toward uniform / toward harmonic; pooled MLE; rank-$2$, rank-$3$ SVD). On all $97$ GIFT-Eval configurations, none significantly beats the frozen $K{=}2$ sample average at $95\%$ confidence under Holm correction; two are significantly worse (Section~\ref{sec:frozen_shape}). With only $K$ complete cycles, extra flexibility adds variance faster than it removes bias. The paper's title follows.

\paragraph{Contributions.}
(i) The rank-1 observation, surveyed on the full GIFT-Eval benchmark (Section~\ref{sec:observation}).
(ii) \emph{Don't learn the shape}: a Holm-corrected negative result on $8$ Shape-learning variants, backed by a Baik--Ben Arous--P\'ech\'e-type under-identification argument (Section~\ref{sec:frozen_shape}, Appendix~\ref{app:bbp}).
(iii) FLAIR (Factored Level And Interleaved Ridge), a closed-form periodic forecaster (one SVD per period candidate, frozen $K{=}2$ Shape, GCV-averaged Ridge on the Level) with a graceful-degeneration cascade (Section~\ref{sec:cascade}) that falls back to Seasonal Naive or plain Ridge using training-window diagnostics alone.

FLAIR is \textbf{Effective} (ties PatchTST: relMASE $0.838$ vs $0.849$), \textbf{Compact} ($P{+}p{+}1$ scalars per series; $28$ for hourly $P{=}24$, $57$ for weekly $P{=}52$), \textbf{Fast} ($22$ min on one CPU core), \textbf{Closed-form} (no GPU, no SGD), and \textbf{Hands-Off} (no per-task hyperparameter search; all constants frozen across the benchmark).

\section{Periodic Series are Rank-1 (median $r_1 = 0.82$)}
\label{sec:observation}

Periodic time series are rank-1 in disguise: on GIFT-Eval, the per-configuration median centered $r_1$ is $0.82$. The reshape that exposes it is mechanical. A series $y_1, \ldots, y_n$ with period $P$ becomes a matrix $\mathbf{M} \in \mathbb{R}^{P \times n_c}$, $n_c = \lfloor n/P \rfloor$.
Row $j$ holds all observations at phase $j$ (e.g., hour $j$ of each day for hourly data with $P{=}24$), and column $i$ holds the $P$ phases of period $i$.
Figure~\ref{fig:rank1} visualises the entire FLAIR pipeline on three real datasets (reshape $\to$ rank-1 structure $\to$ $\mathbf{L}\mathbf{S}^\top$ decomposition $\to$ reconstruction).

Compute the SVD $\mathbf{M} = \sum_k \sigma_k \mathbf{u}_k \mathbf{v}_k^\top$ and define $r_1 = \sigma_1^2 / \sum_i \sigma_i^2$.
Any constant offset $c$ contributes a rank-1 term $c\,\mathbf{1}\mathbf{1}^\top$ that inflates raw $r_1$ unrelated to periodic structure: FLAIR's shifted-positive matrix reaches per-configuration median $r_1 = 0.99$ for this reason ($78\%$ of configs exceed $0.9$; Appendix~\ref{app:rank1_distribution}), a shift artifact, but absorbed by $\mathbf{L}\mathbf{S}^\top$ at zero cost in our factorization.
The honest measure subtracts the per-phase mean across periods (row mean of $\mathbf{M}$, removing both the constant and the seasonal profile).
Row-centered, the per-configuration median $r_1$ on GIFT-Eval is $0.82$ (a descriptive statistic, not used to tune any hyperparameter).
The high end reaches $0.97$ (m4\_quarterly) and $0.91$ (electricity/H), while heterogeneous datasets like loop\_seattle/H and temperature\_rain/D drop below $0.4$ and pull the per-series median to $0.67$.
FLAIR's SVD operates on shifted-positive; centered $r_1$ governs reconstruction error and is reported below.

\begin{proposition}[Slow amplitude $\Rightarrow$ rank-1, error $\leq 2 C_A P / \min|A|$]
\label{prop:rank1}
Let $y_t = A(t) \cdot s(t \bmod P)$ where $s$ is a periodic shape with period $P$ and $A$ is Lipschitz-continuous (per index, $\Delta t = 1$) with constant $C_A$.
Under the slow-variation regime $C_A P \leq \tfrac{1}{2}\min_i |A(iP)|$ (which precludes $A$ from changing sign within a period), the reshaped matrix $M_{j,i} = y(iP + j)$ satisfies
\begin{equation}
    \frac{\|\mathbf{M} - \mathbf{L}\mathbf{S}^\top\|_F}{\|\mathbf{M}\|_F} \;\leq\; \frac{2\,C_A \, P}{\min_i |A(iP)|},
    \label{eq:rank1_bound}
\end{equation}
where $L_i = A(iP)$ and $S_j = s(j)$. The bound vanishes as $C_A P / \min_i|A(iP)| \to 0$, recovering perfect rank-1.
\end{proposition}

\begin{proof}[Proof sketch]
The Lipschitz condition gives $|A(iP+j) - A(iP)| \leq C_A j < C_A P$, so each residual satisfies $|M_{j,i} - L_i S_j| \leq C_A P \cdot |s(j)|$.
Summing: $\|\mathbf{E}\|_F^2 \leq C_A^2 P^2 \cdot n_c \|\mathbf{s}\|^2$.
Under the slow-variation assumption $C_A P \leq \tfrac{1}{2}\min_i|A(iP)|$, the denominator satisfies $\|\mathbf{M}\|_F^2 \geq \tfrac{1}{4}[\min_i A(iP)]^2 \cdot n_c \|\mathbf{s}\|^2$, yielding Eq.~\ref{eq:rank1_bound}.
See Appendix~\ref{app:proof_rank1} for the full proof.
\end{proof}

The bound is loose in practice (per-config median centered $r_1 = 0.82$, well below 1 but dominant) and vacuous on series with spikes, missing data, or near-zero amplitude where $C_A$ is unbounded or $\min |A| \to 0$ (e.g.\ traffic, nn5).

Rank-1 dominance is necessary, not sufficient: FLAIR fails when $n_c$ is too small or the Shape drifts (Section~\ref{sec:failure}). The $(r_1, n_c)$ phase diagram (Figure~\ref{fig:phase}, calibrated on Chronos and validated on GIFT-Eval/Monash/M5 in Section~\ref{sec:experiments}) is FLAIR's decision rule. Section~\ref{sec:sufficiency} formalizes the resulting compression.

\section{Method: Forecast the Level, Freeze the Shape}
\label{sec:method}

How many textbook tools survive the rank-1 commitment? Surprisingly few.

FLAIR is one equation: $\hat{y}_h = \hat{L}_{\lceil h/P \rceil} \cdot S_{h \bmod P}$. Each component below is textbook; the thesis is which ones survive the commitment to rank-1. Table~\ref{tab:constants} lists every internal constant (all global, frozen across the benchmark; no per-dataset tuning). Series are first shifted to positivity ($y_t \leftarrow y_t + \max(1 - \min_t y_t, 1)$) to enable Box-Cox.

\subsection{Pick $P$ by BIC}

Period selection becomes a model-selection problem the moment you commit to rank-1: each $P$ defines a different rank-1 hypothesis, and BIC arbitrates. FLAIR picks $P$ from a calendar-based candidate set $\mathcal{C}$ ($\{24, 168\}$ for hourly, $\{7, 30\}$ for daily, $\{12\}$ for monthly, etc.). For each $P_c \geq 2$, reshape $y_1, \ldots, y_n$ and compute the rank-1 residual $\text{RSS}_1 = \sum_{i \geq 2} \sigma_i^2$ from the singular values $\sigma_i$. A $P{=}1$ null (mean $+$ noise: $\text{RSS} = n \cdot \text{Var}(y)$, $k{=}1$) competes via BIC:
\begin{equation}
    P^* = \arg\min_{P_c \in \{1\} \cup \mathcal{C}} \Big[ n \log\!\big(\text{RSS}(P_c) / n\big) + k(P_c) \log n \Big],
    \label{eq:bic}
\end{equation}
where $k(P_c) = P_c + n_c - 1$ for $P_c \geq 2$ and $k(1) = 1$.
This is the BIC instantiation of the MDL principle: pick the period whose rank-1 reshape gives the shortest two-part description.
When $n_c < 3$ or the DoF guard fires ($n_{\text{train}} < 2p$, where $p$ is the Ridge feature count), FLAIR sets $P{=}1$ (Ridge on raw series).
BIC selects $P$; getting it right matters: hourly $P \pm 1$ increases MASE $6$--$8\times$ (Appendix~\ref{app:period_misspec}).

\subsection{Why $K{=}2$ is Robust}
\label{sec:shape_estimation}

Reshape the shifted series into $\mathbf{M} \in \mathbb{R}^{P \times n_c}$ (Figure~\ref{fig:rank1}c--e). The Level is column sums $L_i = \sum_j M_{j,i}$; the Shape averages within-period proportions over the $K{=}2$ most recent periods, renormalized:
\begin{equation}
    S_j = \frac{\tilde S_j}{\sum_{j'} \tilde S_{j'}}, \quad \tilde S_j = \frac{1}{K} \sum_{k=n_c-K+1}^{n_c} \frac{M_{j,k}}{L_k}, \quad j = 1, \ldots, P.
    \label{eq:shape}
\end{equation}

FLAIR uses a single frozen Shape $S \in \Delta^{P-1}$, consistent with the BBP sub-criticality of the rank-1 residual's second singular vector (Section~\ref{sec:theory}). Freezing breaks when the within-window Shape drifts; Section~\ref{sec:failure} documents this empirically and Appendix~\ref{app:phase_selector} turns the $(r_1, n_c)$ phase diagram into a pre-training-window diagnostic.

\subsection{Level: Prior-Centered Ridge}
\label{sec:level}

The Level $L_1, \ldots, L_{n_c}$ (column sums of $\mathbf{M}$) is a scalar problem.
Box-Cox ($\lambda \in [0,1]$ by MLE) stabilizes variance, secondary seasonality is divided out, and the Level is centered at its last value~\citep{zeng2023dlinear}: $L^{\text{innov}}_i = L^{(\lambda)}_i - L^{(\lambda)}_{n_c}$.

\paragraph{Prior-centered Ridge.}
Standard Ridge has the wrong prior for a Level. We fix this in one line. A neural Level forecaster has thousands of parameters; ours has four. Eq.~\ref{eq:ridge_level} is the entire forecaster:
\begin{equation}
    L^{\text{innov}}_i = \beta_0 + \beta_1 (i/n_c) + \beta_2 \, L^{\text{innov}}_{i-1} + \beta_3 \, L^{\text{innov}}_{i-\mathrm{sec}} + \varepsilon_i,
    \label{eq:ridge_level}
\end{equation}
where $L^{\text{innov}}_{i-\mathrm{sec}}$ is the secondary-period lag (e.g., same hour last week); dropped when absent ($p{=}3$).
Standard Ridge shrinks $\beta_2 \to 0$ (white-noise prior); for slowly-varying Levels, $\beta_2 \approx 1$, so a random-walk prior is correct. GCV soft-averaging (Eq.~\ref{eq:ridge_sa}) adapts $\alpha$ for mean-reverting cases.
We center the penalty at $\boldsymbol{\beta}^* = (0, 0, 1, 0)^\top$ (with $\beta^*_3 = 0$ when the secondary-period lag is present; the fourth coordinate is dropped otherwise):
\begin{equation}
    \min_{\boldsymbol{\beta}} \|\mathbf{y} - \mathbf{X}\boldsymbol{\beta}\|^2 + \alpha \|\boldsymbol{\beta} - \boldsymbol{\beta}^*\|^2.
    \label{eq:prior_centered}
\end{equation}
Substituting $\boldsymbol{\delta} = \boldsymbol{\beta} - \boldsymbol{\beta}^*$ gives standard Ridge on a differenced target (Shape and $\lambda$ are structural, not counted in forecast DoF):
\begin{equation}
    \Delta L^{\text{innov}}_i = \delta_0 + \delta_1 (i/n_c) - \delta_2 \, L^{\text{innov}}_{i-1} + \delta_3 \, L^{\text{innov}}_{i-\mathrm{sec}} + \varepsilon_i,
    \quad \delta_2 = 1 - \beta_2,\ \delta_3 = \beta_3.
    \label{eq:diff_target}
\end{equation}
Under H\"{o}lder(2) (bounded second difference of $L_i$; cycle-aggregation gives H(2) even when $y_t$ has spikes), all $\delta$'s are near zero; isotropic Ridge is the correct shrinkage. The choice $\beta^*_2 = 1$ is the unique scale-invariant choice consistent with the H(2) random-walk fixed point ($\beta_2 \to 1$ as $\sigma_\varepsilon \to 0$); any other value injects scale-dependent bias.

FLAIR computes 25 Ridge solutions ($\alpha \in [10^{-4}, 10^4]$, log-spaced) from one SVD and averages with softmax weights on generalized-cross-validation (GCV) scores~\citep{golub1979gcv}:
\begin{equation}
    \hat{\boldsymbol{\delta}} = \sum_{k=1}^{25} w_k \boldsymbol{\delta}_k, \quad w_k \propto \exp\!\big({-}(\text{GCV}_k - \text{GCV}_{\min}) / \text{GCV}_{\min}\big).
    \label{eq:ridge_sa}
\end{equation}

\paragraph{Reconstruction and fallback.}
\label{sec:cascade}
Forecasts reconstruct via Eq.~\ref{eq:core}: $\hat{y}_h = \hat{L}_{\lceil h/P\rceil} \cdot S_{h \bmod P}$ minus the location shift, with damped trend $\phi = \max(\hat\rho_1(\Delta L), 0) \wedge (1{-}\epsilon)$ (conservative, non-MLE; bounds long-horizon variance, Appendix~\ref{app:cascade_proof}). Plain Ridge on raw $y_t$ when $\hat P = 1$ or the DoF guard fires; uncertainty quantification (leverage bootstrap, phase noise) and the full three-branch cascade are in Appendix~\ref{app:cascade_proof}.

\section{Why Freezing the Shape Works}
\label{sec:theory}

Once a periodic series is aligned and nearly rank-1, the only learnable object is the Level. This section explains why the compression helps and why Shape learning usually fails.

\subsection{Rank-1 Turns $P$ Forecasts into One}
\label{sec:sufficiency}

Predicting 336 hourly values looks like a 336-dimensional problem.
After rank-1 decomposition: 14 Level steps times a fixed Shape.
The LSR1 (Locally Stationary Rank-1) model: $M_{j,i} = L(i/n_c)\,S_j + \varepsilon_{j,i}$ with smooth Level and i.i.d.\ noise. Theorem~\ref{thm:rank1_sufficiency} below quantifies the gain. The $K{=}2$ unknown-Shape case is in Appendix~\ref{app:joint_bound}.

\begin{theorem}[Rank-1 advantage over a phase-independent local-linear baseline]
\label{thm:rank1_sufficiency}
Let $M_{j,i} = L(i/n_c)\,S_j + \varepsilon_{j,i}$ with $L \in \mathrm{H\ddot{o}lder}(2)$ on $[0,1]$, $\mathbf{S} \in \Delta^{P-1}$, and $\varepsilon_{j,i}$ i.i.d.\ sub-Gaussian$(\sigma^2)$.
Compare two strategies for forecasting the next cycle $(L(1)S_0,\ldots,L(1)S_{P-1})$ \emph{with the true Shape $\mathbf{S}$ known}:
\textbf{(A, baseline)}~estimate each phase separately via local linear regression on $n_c$ points (phase-independent, no cross-phase pooling);
\textbf{(B, rank-1)}~aggregate into one Level series $\tilde{L}_i = \sum_j M_{j,i}$, estimate $L(1)$, and reconstruct $\hat L(1) \cdot \mathbf{S}$.
Then there exist constants $C_A, C_B > 0$ such that for every $L$ in the class and every $\mathbf{S} \in \Delta^{P-1}$,
\begin{equation}
    \mathrm{MSE}^{(A)} \leq C_A \cdot \frac{P\sigma^2}{n_c^{4/5}}, \qquad
    \mathrm{MSE}^{(B)} \leq C_B \cdot \frac{P\|\mathbf{S}\|_2^2\,\sigma^2}{n_c^{4/5}}.
    \label{eq:rank1_rates}
\end{equation}
Within this phase-independent baseline class, Strategy~(B)'s upper bound improves on Strategy~(A)'s by the factor $\|\mathbf{S}\|_2^2 \leq 1$ (simplex non-negativity gives the upper bound; Cauchy-Schwarz gives the lower bound $1/P$ at uniform; $\approx 2/P$ for solar with night-time zeros). The comparison is between the two upper bounds; tight matching lower bounds against phase-independent baselines are standard~\citep{tsybakov2009nonparametric}, so the rate gap survives at the minimax level on this class.
Theorem~\ref{thm:joint_rank1} in Appendix~\ref{app:joint_bound} removes the known-Shape assumption by plugging in the sample-proportion estimate of Eq.~\ref{eq:shape}.
\textbf{Scope.} The bound is restricted to the phase-independent local-linear baseline class. Methods that pool across phases without the rank-1 assumption (joint Gaussian-process regression, low-rank-$r$ with $r \geq 2$, or spectral estimators on the reshaped matrix) lie outside this class and may achieve intermediate rates; the theorem is silent about whether they close the $P$-fold gap. The empirical 8-variant sweep (Section~\ref{sec:frozen_shape}) supplies the direct evidence that rank-$r \in \{2,3\}$ SVD variants do not.
\end{theorem}

\begin{proof}[Proof sketch]
Strategy~(A) runs $P$ independent regressions, each at the minimax rate $O(\sigma^2 n_c^{-4/5})$ for H\"{o}lder(2) boundary estimation with Epanechnikov kernel and bandwidth $h \asymp n_c^{-1/5}$~\citep{fan1996local}.
Sum over $P$ phases.
Strategy~(B) aggregates $P$ noisy observations per period into one Level value, with column-sum noise variance $P\sigma^2$.
Local linear regression on the Level at the same rate gives MSE $= O(P\sigma^2 n_c^{-4/5})$; the forecast at phase $j$ scales this by $S_j^2$, so summed over phases the forecast MSE is $\|S\|_2^2 \cdot O(P\sigma^2 n_c^{-4/5})$.
Full proof in Appendix~\ref{app:proof_rank1_sufficiency}.
\end{proof}

For hourly electricity ($P{=}24$, 2:1 day/night) the reduction relative to Strategy~(A) is ${\sim}22$-fold within this baseline class; an LSR1 minimax bracket confirms tightness, and FLAIR's prior-centered Ridge is rate-competitive with local-linear (Appendix~\ref{app:minimax}, \ref{app:stl} for additive STL counterexample).

\subsection{Only the Level Needs Learning}
\label{sec:minimax}

The feature basis $\mathbf{x}_i = [1, i/n_c, L^{\mathrm{innov}}_{i-1}, L^{\mathrm{innov}}_{i-\mathrm{sec}}]^\top$ (intercept, linear drift, AR(1), secondary-period lag) is a modelling choice, common to FLAIR's ablations and prior seasonal-Ridge work.
\emph{Given this basis}, the rank-1 model pins the prior mean.
Under LSR1 with slowly-varying Level, write $L^{(\lambda)}_i = L^{(\lambda)}_{i-1} + \eta_i$ where $\eta_i$ is a small innovation (zero under exact-quadratic $L$, $O(\|R\|_\infty)$ for general H\"{o}lder-2).
Recentering at the last value $L^{\mathrm{innov}}_i = L^{(\lambda)}_i - L^{(\lambda)}_{n_c}$ and expanding Eq.~\ref{eq:ridge_level},
\[
    L^{\mathrm{innov}}_i \;=\; \beta_0 + \beta_1 (i/n_c) + \beta_2\, L^{\mathrm{innov}}_{i-1} + \varepsilon_i,
\]
the true coefficients under LSR1 are $(\beta_0, \beta_1, \beta_2, \beta_3) = (0, 0, 1, 0)$: no intercept (centered), no linear drift (NLinear anchor), persistence one (random walk), no cross-period lag (rank-1 residual is noise).
Placing a Gaussian prior $\boldsymbol{\beta} \sim \mathcal{N}(\boldsymbol{\beta}^*, \tau^2 I)$ on the coefficients and reading off the MAP with regularization $\alpha = \sigma^2 / \tau^2$ gives exactly Eq.~\ref{eq:prior_centered}.
The feature-basis choice is orthogonal to the prior-center choice; LSR1 pins the latter, not the former.

As $\alpha \to \infty$ the MAP collapses to the prior and the forecast reverts to $L_{n_c} \cdot S_j$, so Seasonal Naive sits at the right end of the regularization path and the Ridge Level can only relax the forecast towards it.
On smooth Levels, Ridge bias is bounded by the H\"{o}lder-2 remainder $\|R\|_\infty$ rather than the $n_c^{-4/5}$ minimax rate, so for typical $n_c \in [10, 100]$ the variance rate $O(\sigma_e^2 / n_c)$ is competitive with local linear $O(\sigma_e^2 n_c^{-4/5})$ (Appendix~\ref{app:ridge_bound}, Corollary~\ref{cor:competitive}).
GCV soft-averaging (Eq.~\ref{eq:ridge_sa}) selects $\alpha$ adaptively~\citep{andrews1991asymptotic}, so the prior strength is data-driven rather than hyperparameter-tuned.

\subsection{Why Shape Learning Fails}
\label{sec:frozen_shape}

\begin{figure}[t]
    \centering
    \includegraphics[width=0.82\textwidth]{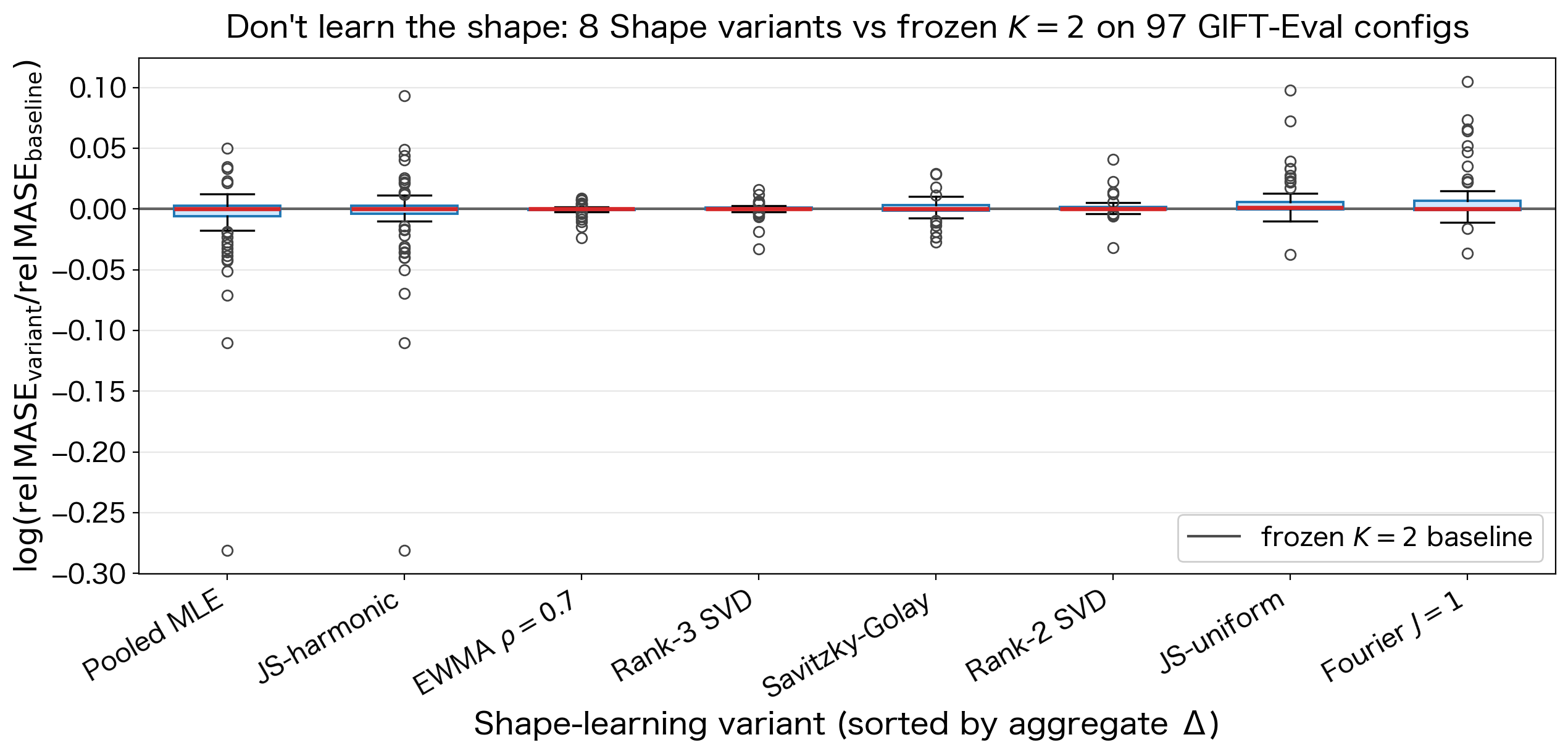}
    \caption{Don't learn the shape: 8 Shape-learning variants sorted by per-config $\Delta$\,relMASE against the frozen $K{=}2$ baseline on the full 97-configuration benchmark (N${=}200$ forecast samples). Box = interquartile range of the per-config log-ratio; red bar = median; zero line = baseline. No variant improves on the baseline at $95\%$ confidence under Holm correction; two on the right edge (JS-uniform, Fourier $J{=}1$) are significantly worse.}
    \label{fig:frozen}
\end{figure}

Two families fail, two reasons.

\paragraph{(A) SVD-based variants: BBP subcriticality.}
Rank-2 and rank-$r$ SVD try to recover the second singular vector $\mathbf{u}_2$ of $\mathbf{M}$ and add it to the approximation.
For this family, a BBP phase-transition argument on a spiked-rectangular surrogate characterises when recovery is possible; the theorem below states the surrogate result, and the bridge to real GIFT-Eval data is heuristic (see the caveat after the theorem).

\paragraph{(B) Smoothing, shrinkage, and averaging variants: MVUE of the pooled sample proportion.}
The six non-SVD variants (EWMA, Fourier $J{=}1$, Savitzky-Golay, JS toward uniform, JS toward first harmonic, pooled MLE) do \emph{not} touch $\mathbf{u}_2$; they reshape how the $K$ within-period proportions are combined into $\hat{S}$.
Under a stylized count-model analogy the pooled sample proportion is admissible (Remark~\ref{thm:shape_optimal}, Appendix~\ref{app:admissibility}); FLAIR's equal-weight $K{=}2$ average agrees with pooled MLE to first order under LSR1 slow-amplitude, so any smoothing/shrinkage variant would have to improve on a constant-factor approximation of an admissible estimator. Empirically, none of the $8$ Shape-learning variants beats the frozen $K{=}2$ baseline at $95\%$ confidence under Holm correction (Figure~\ref{fig:frozen}); the separate $K$-sweep on all 97 configs (Appendix~\ref{app:ksweep}) places $K{=}2$ on a plateau with $K{=}5,\,50,\,100$ (CIs cross zero) and significantly above $K{=}1,\,4,\,10,\,n_c$ (CIs above zero), within a $1\%$ admissible band overall.
Dirichlet-Multinomial shrinkage injects a flat prior and \emph{can} reduce variance under small $K$, but the 97-config sweep shows the Prior-centered Ridge already absorbs this through cross-period lags.

A BBP spiked-rectangular surrogate gives a heuristic, narrative reason (the 8-variant sweep below carries the empirical claim): $58\%$ of GIFT-Eval series fall below the recovery threshold $\beta_k^2/\sigma^2 > \sqrt{P/n_c}$~\citep{baik2005phase} on the shifted-positive matrix (Appendix~\ref{app:bbp}). The empirical 8-variant sweep (Figure~\ref{fig:frozen}, Appendix~\ref{app:frozen_table}) carries the claim; BBP is narrative.
Context conditioning via day-of-week Dirichlet indexing was the one natural exception we tested (known index, outside the BBP argument); on the 97-config benchmark it was marginally harmful at aggregate, so FLAIR uses a single frozen Shape. Three falsifiable predictions follow: FLAIR should win at high centered $r_1$ and large $n_c$; additive decomposition should lose under amplitude drift; FLAIR's failures should concentrate in short-cycle or shape-drift regimes (Section~\ref{sec:experiments}).

\section{Experiments}
\label{sec:experiments}

The theory makes three predictions: FLAIR wins where centered $r_1$ is high and $n_c$ is not small; additive decomposition loses when amplitude varies; and failures concentrate in short-cycle or shape-drift regimes. We test each below.

\subsection{Setup}

We evaluate on \textbf{GIFT-Eval}~\citep{gift_eval_2024} (97 configurations, 23 datasets, 7 domains, 5-minute to yearly) and \textbf{Chronos zero-shot}~\citep{ansari2024chronos} (25 datasets via FEV~\citep{fev2025}). Metrics: relMASE / relCRPS as geomean of per-config Seasonal-Naive ratios. FLAIR uses $n_{\text{samples}}{=}200$ (Monte Carlo error $\pm 0.3\%$; Appendix~\ref{app:paired_bootstrap}). Baselines: foundation, deep-learning (PatchTST, DLinear, N-BEATS, TFT, iTransformer, DeepAR~\citep{salinas2020deepar}), and statistical (AutoARIMA, AutoTheta, ETS, Prophet).

\paragraph{Aggregate.}
Table~\ref{tab:gift_eval} reports relMASE $= 0.838$ on all $97$ configurations, the anchor for the paper. The $93$-configuration subset restricted to configs with both centered $r_1$ and Chronos-Bolt-Base scores (4 configs dropped by data availability, not performance; Appendix~\ref{app:winloss_stratified} lists them) gives FLAIR $0.830$ vs Bolt $0.806$. All per-$r_1$ analyses including the crossover band $[0.77, 0.90]$ use this 93-config subset.

\subsection{RQ1: Does FLAIR Match Per-Dataset Deep Learning?}

\begin{figure}[t]
    \centering
    \includegraphics[width=\textwidth]{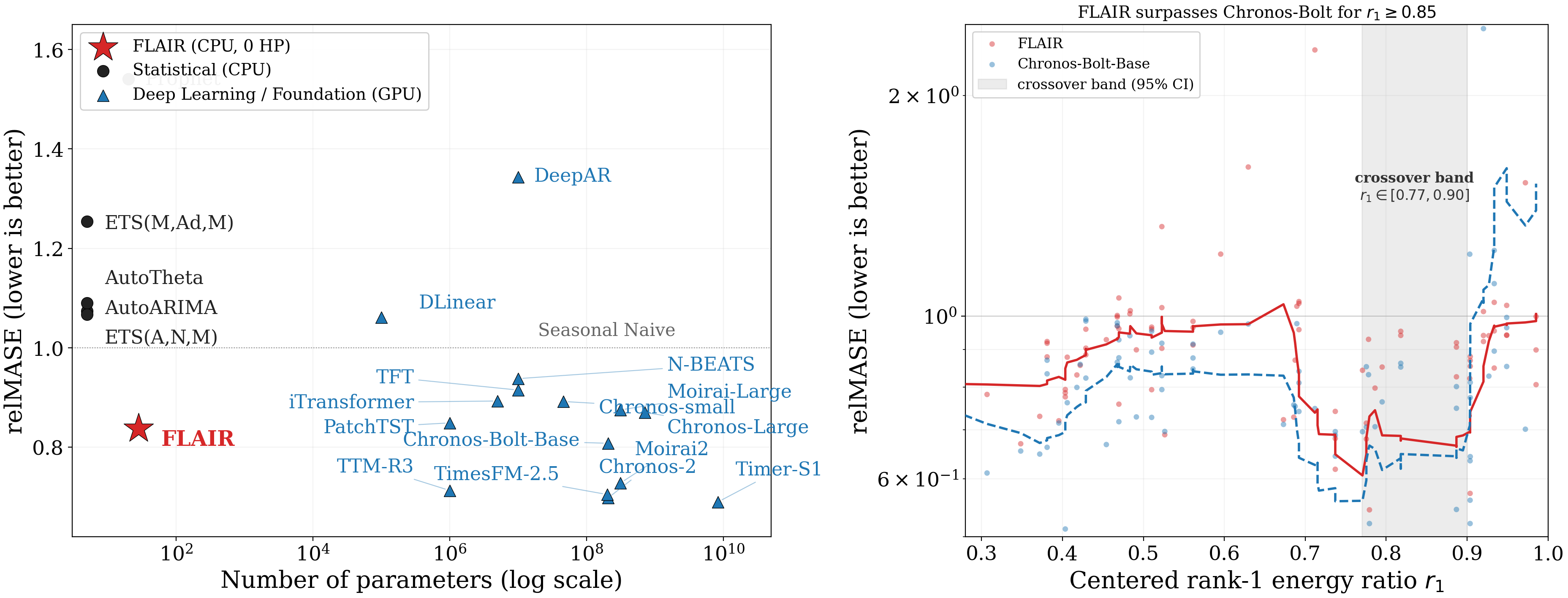}
    \caption{\textbf{Left: FLAIR is Pareto-optimal at the small-model frontier on GIFT-Eval} ($97$ configurations, geometric-mean relMASE vs.\ parameter count; foundation models lead aggregate but at $10^6$-$10^8 \times$ FLAIR's parameter count). The parameter axis mixes units: trained weights for deep-learning and foundation models, per-series estimate counts ($P{+}p{+}1$ for FLAIR, ${\sim}5$ for AutoARIMA/AutoTheta/ETS) for the classical baselines; marker shape indicates which. \textbf{Right:} per-configuration relMASE sorted by centered rank-1 energy $r_1$. FLAIR's rolling geomean falls below Chronos-Bolt-Base's within the crossover band $r_1 \in [0.77, 0.90]$ (bootstrap $95\%$ CI on the first crossover). The two diagnostics $(r_1, n_c)$ that set the regime are computable from the training window alone.}
    \label{fig:hero}
\end{figure}

\begin{table}[t]
\centering
\caption{\textbf{FLAIR ties PatchTST on relCRPS ($0.587$ vs $0.587$) at $P{+}p{+}1$ scalars per series, no GPU, no pre-training.} Foundation models lead aggregate GIFT-Eval at $10^6$--$10^8 \times$ FLAIR's parameter count and $27$--$100$B-observation pre-training corpora; FLAIR occupies a distinct point on the cost-accuracy frontier (Figure~\ref{fig:hero}, left). relMASE and relCRPS: geometric mean relative to Seasonal Naive (lower is better).}
\label{tab:gift_eval}
\small
\begin{tabular}{llcccc}
\toprule
Model & Type & relMASE & relCRPS & Params & GPU \\
\midrule
Timer-S1$^\ddagger$ & Foundation (MoE) & 0.69 & 0.49 & 8.3B & Yes \\
Chronos-2 & Foundation & 0.698 & 0.485 & 205M & Yes \\
TimesFM-2.5 & Foundation & 0.705 & 0.490 & 200M & Yes \\
TTM-R3$^\dagger$ & Foundation & 0.713 & 0.511 & ${\sim}$1M & Yes \\
Moirai2 & Foundation & 0.728 & 0.516 & 311M & Yes \\
Chronos-Bolt-Base & Foundation & 0.808 & 0.574 & 205M & Yes \\
\textbf{FLAIR} & \textbf{Statistical} & \textbf{0.838} & \textbf{0.587} & $P{+}p{+}1^*$ & \textbf{No} \\
PatchTST & Deep Learning & 0.849 & 0.587 & ${\sim}$1M & Yes \\
Chronos-Large & Foundation & 0.870 & 0.647 & 710M & Yes \\
Moirai-Large & Foundation & 0.875 & 0.599 & 311M & Yes \\
Chronos-small & Foundation & 0.892 & 0.663 & 46M & Yes \\
iTransformer & Deep Learning & 0.893 & 0.620 & ${\sim}$5M & Yes \\
TFT & Deep Learning & 0.915 & 0.605 & ${\sim}$10M & Yes \\
N-BEATS & Deep Learning & 0.938 & 0.816 & ${\sim}$10M & Yes \\
Seasonal Naive & Baseline & 1.000 & 1.000 & 0 & No \\
DLinear & Deep Learning & 1.061 & 0.846 & ${\sim}$0.1M & Yes \\
ETS(A,N,M) & Statistical & 1.067 & 0.833 & ${\sim}$5 & No \\
AutoARIMA & Statistical & 1.074 & 0.912 & ${\sim}$5 & No \\
AutoTheta & Statistical & 1.090 & 1.244 & ${\sim}$5 & No \\
ETS(M,Ad,M) & Statistical & 1.254 & 1.088 & ${\sim}$5 & No \\
DeepAR & Deep Learning & 1.343 & 0.853 & ${\sim}$10M & Yes \\
Prophet & Statistical & 1.540 & 1.061 & ${\sim}$20 & No \\
\bottomrule
\end{tabular}\\[2pt]
{\footnotesize $^*$$P{+}p{+}1$ quantities estimated per series ($P$-dim Shape, $p \in \{3,4\}$ Ridge features, and Box-Cox $\lambda$); effective forecast DoF is $p{=}3$--$4$ (Ridge features only). $^\ddagger$Timer-S1 is a Mixture-of-Experts model with $0.75$B activated parameters per token (out of $8.3$B total); zero-shot. $^\dagger$TTM-R3 is fine-tuned per dataset (zero-shot TTM-R3 scores $0.724$/$0.520$); other foundation models are zero-shot. FLAIR, AutoARIMA, AutoTheta, ETS, Prophet are fit per-series; DLinear, PatchTST per-dataset. We exclude leaderboard entries with \texttt{testdata\_leakage: Yes} (TimesFM v1, Lag-Llama; Appendix~\ref{app:chronos_table}). AutoTBATS errors on $7$ zero-heavy configs; reported on Chronos benchmark (Appendix~\ref{app:mstl_tbats}). Foundation-model numbers from the public GIFT-Eval leaderboard~\citep{gift_eval_2024} under each submitter's protocol; FLAIR re-run locally. Per-$r_1$ FLAIR-vs-Chronos-Bolt win/loss in Appendix~\ref{app:winloss_stratified}.}
\end{table}

Table~\ref{tab:gift_eval} places FLAIR against post-2024 baselines: Timer-S1~\citep{timer-s1}, Chronos-2, TimesFM-2.5, TTM-R3, and Moirai2 lead the aggregate at $0.69$--$0.73$ relMASE with $1$M--$8.3$B trained parameters and GPU pre-training (on $27$--$100$B time-series observations); FLAIR reaches $0.838$ / $0.587$ at $P{+}p{+}1$ quantities per series on a CPU, occupying a distinct point on the cost-accuracy frontier (Figure~\ref{fig:hero}, left). The aggregate gap to PatchTST is below detectable resolution at $n{=}97$ (paired-bootstrap CI on the ratio includes $1.0$ for both relMASE and relCRPS), per-config Wilcoxon splits $43$:$54$ (FLAIR : PatchTST; Appendix~\ref{app:paired_bootstrap}), and the geomean is $22\%$ ahead of the best classical baseline. Internal constants are tabulated in Table~\ref{tab:constants}, frozen across all $97$ configurations.
Within the $r_1$ band $[0.77, 0.90]$ (bootstrap $95\%$ CI on the first crossover, Appendix~\ref{app:crossover}), FLAIR's rolling-geomean relMASE falls below Chronos-Bolt-Base's (Figure~\ref{fig:hero}, right); the two trade wins by configuration and neither dominates. Per-bin head-to-head against PatchTST is mixed: FLAIR wins $62.5\%$/$66.7\%$ of configs in the two highest-$r_1$ bins ($r_1 \geq 0.79$) and loses on $39\%$--$54\%$ in the lower-$r_1$ bins (Appendix~\ref{app:winloss_stratified}); the aggregate tie reflects geomean weighting and a handful of high-impact BizITOBS configurations. Appendix~\ref{app:phase_selector} gives the $(r_1, n_c)$ phase diagram.

\begin{figure}[t]
    \centering
    \includegraphics[width=0.78\textwidth]{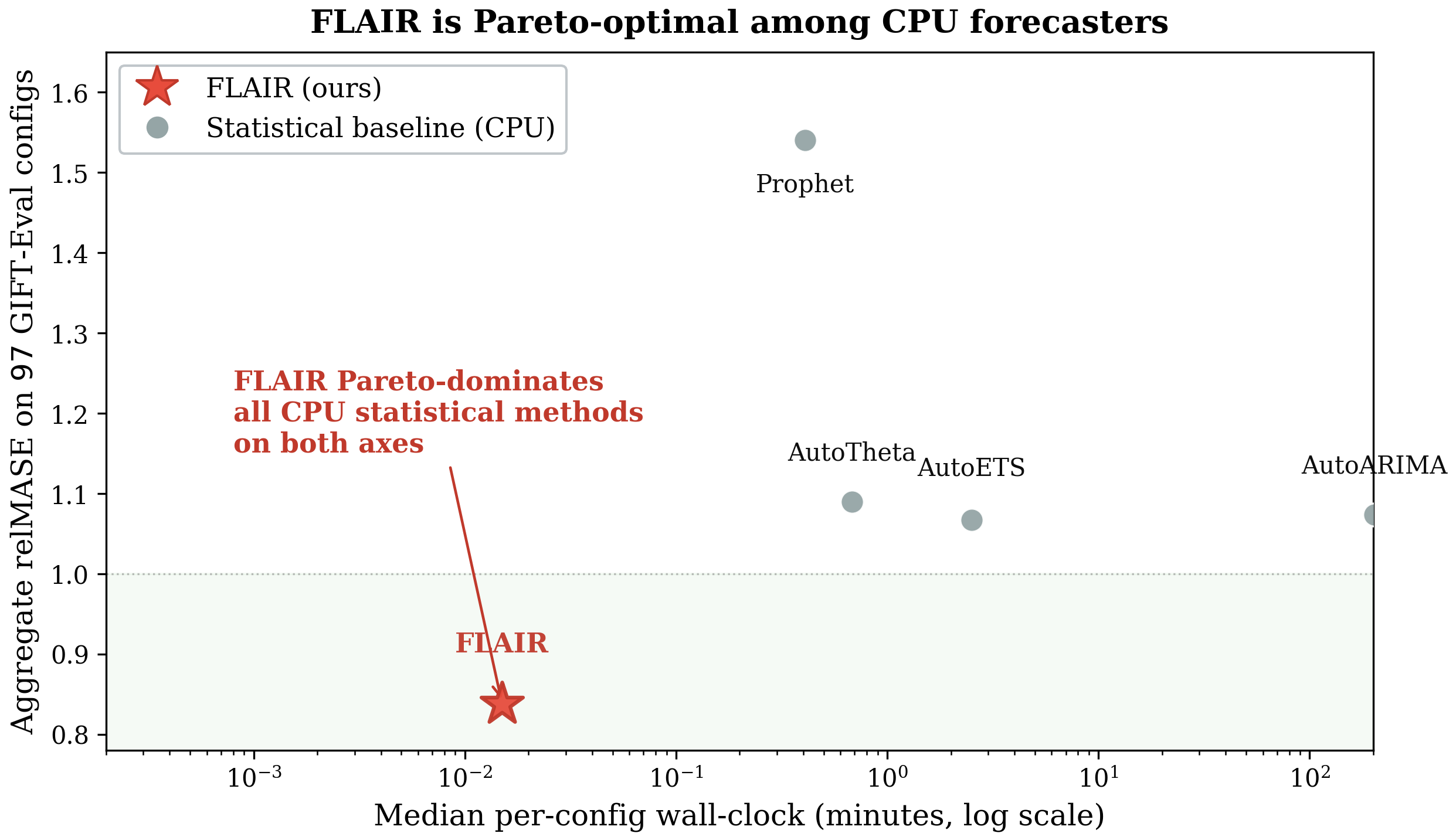}
    \caption{\textbf{FLAIR Pareto-dominates the reported CPU statistical baselines (AutoTheta, Prophet, AutoETS, AutoARIMA) on accuracy and wall-clock, at $0.9$\,s/config and $P{+}p{+}1$ scalars per series.} X-axis: median per-config wall-clock; FLAIR/Prophet measured on a 2024 MacBook Pro (M4 Pro) one core; AutoTheta/AutoETS/AutoARIMA estimated from \texttt{statsforecast} per-series fit times at length buckets $\{200, 700, 2000\}$, scaled by $n_{\text{series}}$ and horizon. AutoARIMA exceeded $6$ CPU-hours on long-series configs; the plotted point is a lower bound. Foundation models (Chronos, TimesFM, etc.) are excluded because their wall-clock excludes pretraining; their accuracy is in Table~\ref{tab:gift_eval}.}
    \label{fig:pareto}
\end{figure}

\subsection{RQ2: FLAIR Beats Per-Series Statistical Methods on Chronos}
\label{sec:chronos}

On the $25$-dataset Chronos zero-shot benchmark (Monash-derived), FLAIR reaches aggregate relative MASE $0.678$ and WQL $0.716$, best among per-series methods ($8.9\%$ below AutoTBATS on MASE, $17\%$ below it on WQL, $21.6\%$ below AutoARIMA on MASE). Foundation models pre-trained on a large corpus lead on WQL (Chronos-Bolt-Base $0.624$); the full benchmark table is in Appendix~\ref{app:chronos_table}.

\paragraph{Ablation.}
Prior-centered Ridge ($+2.0\%$ if removed) and the positivity shift ($+1.0\%$) are the only ablation-visible components; the $P{=}1$ BIC null moves the aggregate by $<0.2\%$ but prevents per-series catastrophes (e.g., \texttt{m4\_hourly/H/short} MASE $1.91 \to 1.81$). The full ablation table and shape-variant sweeps live in Appendix~\ref{app:ablation_table}.

\subsection{RQ3: When FLAIR Fails ($19/97$ configs, all diagnosable)}
\label{sec:failure}

FLAIR fails on $19$ of $97$ configurations (relMASE $>1$), and every failure is diagnosable from the training window alone: small $n_c/P$, centered $r_1 < 0.7$, or a structural violation such as solar's night-time zeros. Appendix~\ref{app:failure_diag} gives the config-level diagnostic with recommended fallbacks (Seasonal Naive for small $n_c/P$, STL/ETS for weak periodicity, Croston for intermittent demand); Appendix~\ref{app:guidelines} gives a full decision rule including cross-series hierarchies that FLAIR does not address.

\paragraph{Wall-clock.}
FLAIR completes all $97$ GIFT-Eval configurations in $22$ min on one CPU core of a $2024$ MacBook Pro (median $0.9$\,s/config, max $148$\,s); no GPU, no pre-training, no parallelism. Figure~\ref{fig:pareto} places FLAIR against the reported CPU statistical baselines: FLAIR Pareto-dominates AutoTheta, Prophet, AutoETS, AutoARIMA on both accuracy and wall-clock ($27\times$ faster than Prophet, ${\sim}13{,}000\times$ faster than AutoARIMA; AutoTBATS errors on zero-heavy configs, Appendix~\ref{app:mstl_tbats}). Foundation models (Chronos: $84$B observations, ${\sim}500$ GPU-hours~\citep{ansari2024chronos}) are omitted from Figure~\ref{fig:pareto} and would sit several orders of magnitude to the right of AutoARIMA.

\section{Related Work}
\label{sec:related}

FLAIR draws on three classical families and one modern one. The X-11/X-13ARIMA-SEATS lineage~\citep{shiskin1967x11, dagum1988x11arima, findley1998x12, sax2018seasonal} freezes a multiplicative seasonal index from recent complete cycles; FLAIR's $K{=}2$ frozen proportion shares this move. Structural state-space models~\citep{harvey1989structural,winters1960forecasting} and ETS$(*,*,M)$ with $\gamma{=}0$~\citep{hyndmanetal2008ets} also allow it, while Theta~\citep{assimakopoulos2000theta}, Prophet~\citep{taylor2018prophet}, TBATS~\citep{delivera2011tbats}, and MSTL~\citep{bandara2021mstl} treat the seasonal profile as smooth or slowly-varying. On the Chronos zero-shot benchmark, FLAIR's relMASE is $0.678$ vs AutoTBATS $0.744$ and Prophet $1.540$ at the same one-CPU-core wall-clock (Appendix~\ref{app:chronos_table}). FLAIR adds a $97$-config negative result on adaptive Shape estimators (Section~\ref{sec:frozen_shape}) and an LSR1 rate-optimality argument (Appendix~\ref{app:minimax}).

Matrix-SVD (SSA~\citep{broomhead1986extracting,golyandina2001analysis}), period-reshape neural~\citep{wu2023timesnet, lin2024sparsetsf, oreshkin2020nbeats}, Koopman~\citep{liu2023koopa}, and foundation-scale~\citep{nie2023patchtst,ansari2024chronos,ansari2025chronos2,woo2024moirai,das2024timesfm,ekambaram2024ttm} approaches learn features on the reshaped matrix; FLAIR keeps only the leading rank-1 component.
\section{Conclusion}
\label{sec:conclusion}

Don't learn the Shape. Learn the Level. FLAIR is \textbf{Effective} (relMASE $0.838$ ties PatchTST), \textbf{Compact} ($28$/$57$ scalars per hourly/weekly series), \textbf{Fast} ($22$ min/CPU), \textbf{Closed-form} (SVD + GCV Ridge), \textbf{Hands-Off} (no per-task tuning). Use FLAIR at $r_1 \geq 0.77$, $n_c \geq 10$; else fall back to Seasonal Naive or STL. In the high-rank-1, many-cycle regime, extra flexibility is estimation noise.

\bibliographystyle{plainnat}
\bibliography{references}

\newpage
\appendix
\section{Capability Grid}
\label{app:capability}

\begin{table}[h]
\centering
\footnotesize
\caption{Capability grid: FLAIR matches all four desiderata; competitors miss at least one. $\checkmark$/$\sim$/$\times$ = yes/partial/no. ``Compact'' $= <10^4$ scalars per series; ``Fast (CPU)'' $= <10$\,s per config on one core; ``Closed-form'' $=$ no gradient descent, no pre-training.}
\label{tab:capability}
\begin{tabular}{lcccc}
\toprule
Method & Effective & Compact & Fast (CPU) & Closed-form \\
\midrule
Foundation models (Chronos, TimesFM, Moirai) & $\checkmark$ & $\times$ & $\times$ & $\times$ \\
Deep learning (PatchTST, DLinear, etc.) & $\checkmark$ & $\sim$ & $\times$ & $\times$ \\
Statistical (AutoARIMA, ETS, Prophet) & $\times$ & $\checkmark$ & $\sim$ & $\checkmark$ \\
\textbf{FLAIR} (this paper) & $\checkmark$ & $\checkmark$ & $\checkmark$ & $\checkmark$ \\
\bottomrule
\end{tabular}
\end{table}

\section{Implementation Details}
\label{app:impl_details}

FLAIR exposes \emph{zero user-facing hyperparameters}: the user passes a series, a horizon, and a sample count.
The constants below are structural defaults; they are fixed and not tuned per-dataset.
Table~\ref{tab:constants} lists them for reproducibility.

\begin{table}[h]
\centering
\caption{FLAIR internal constants. All values are fixed across all 97 GIFT-Eval configurations, 25 Chronos zero-shot datasets, and 32 Monash long-term forecasting tasks used in our experiments.}
\label{tab:constants}
\small
\begin{tabular}{llp{7.5cm}}
\toprule
Symbol / name & Value & Role \\
\midrule
\texttt{\_SHAPE\_K} & 2 & Number of recent periods averaged to form $S$ (Eq.~\ref{eq:shape}); $K \in \{1, 2, 4, 5, 10, 50, 100, n_c\}$ on all $97$ configurations places the aggregate relMASE in a $1.03\%$ range with $K{=}2$ mid-range (Appendix~\ref{app:ksweep}). \\
\texttt{\_PHASE\_NOISE\_K} & 50 & Trailing window used to build the rank-1 residual $\mathbf{E}$ for scenario-coherent phase sampling. \\
\texttt{\_MIN\_COMPLETE} & 3 & Minimum $n_c$ for the non-degenerate $L \!\times\! S$ decomposition; below this FLAIR falls back to $P{=}1$. \\
\texttt{\_MAX\_COMPLETE} & 500 & Cap on $n_c$ (memory / speed guard); only trailing 500 periods are used for Level fitting. \\
\texttt{\_MIN\_POSITIVE\_FOR\_BC} & 10 & Minimum positive observations required to estimate Box-Cox $\lambda$; otherwise $\lambda{=}1$ (identity). \\
\texttt{\_BC\_EXP\_CLIP} & 30 & Clip on the inverse Box-Cox exponent to avoid overflow on extreme $\lambda \to 0$ forecasts. \\
\texttt{\_N\_ALPHAS} & 25 & Ridge GCV soft-average grid size (log-spaced $\alpha \in [10^{-4}, 10^4]$, Eq.~\ref{eq:ridge_sa}). \\
\texttt{\_LEVEL\_NOISE\_MODE} & bootstrap & Bootstrap LOOCV residuals (parametric Student-$t$ only if $n_{\text{loo}}\!<\!4$). \\
\midrule
Ridge prior center $\boldsymbol\beta^*$ & $(0,0,1,0)^\top$ & Prior on $\beta_2{=}1$ (random-walk Level); Eq.~\ref{eq:prior_centered}. \\
DoF guard & $n_{\text{train}} \geq 2p$ & Drops lag features if training sample is insufficient for stable LOOCV. \\
Phase deflation & $1/\sqrt{1{+}h_{\text{test}}}$ & Offsets the conformal leverage factor $\sqrt{1+h_\text{test}}$ to prevent double-counting variance at long horizons. \\
\bottomrule
\end{tabular}
\end{table}

Period candidates $\mathcal{C}$ are derived from \texttt{freq}: hourly gets $\{24, 168\}$, daily $\{7, 30\}$, monthly $\{12\}$, quarterly $\{4\}$, weekly $\{52\}$, 30-minute $\{48, 336\}$, etc.
The $P{=}1$ null model (Eq.~\ref{eq:bic}) competes with every $P_c \in \mathcal{C}$ under a common BIC, so non-periodic series are automatically routed to plain Ridge without a separate detector.

\section{When FLAIR Wins: Centered-$r_1$ and $n_c$ Diagnostics}
\label{app:crossover}

Figure~\ref{fig:hero} (right panel) already plots per-configuration relMASE vs.\ centered $r_1$ with the crossover band $[0.77, 0.90]$ (bootstrap $95\%$ CI). This appendix adds the stratified win/loss table, the operating-regime phase diagram, the leave-one-dataset-out selector, and a short ETS comparison.

\subsection{Per-$r_1$ Win/Loss Table}
\label{app:winloss_stratified}

Table~\ref{tab:winloss} reports per-method geomean relMASE and FLAIR's head-to-head win rates on the $93$ of $97$ GIFT-Eval configurations where all three methods produced comparable scores and centered $r_1$ is available. FLAIR's aggregate relMASE and relCRPS against PatchTST are not significantly different ($50.5\%$ relMASE win rate, $55.9\%$ relCRPS win rate; paired-bootstrap CI on the ratio does not exclude PatchTST). In the $r_1 > 0.9$ bin, FLAIR wins $61.9\%$ of the head-to-heads and leads in geomean ($0.866$ vs $1.216$, median $0.923$ vs $0.964$). The geomean gap is outlier-inflated by a handful of BizITOBS configurations where Bolt's relMASE reaches $2.7$--$3.9$; the median $4\%$ gap is the typical-config figure. Foundation-model wins concentrate below $r_1 = 0.6$, where the rank-1 assumption is weakest.

\begin{table}[h]
\centering
\caption{Per-method relMASE (geometric mean across configs in each bin, consistent with Table~\ref{tab:gift_eval}) and FLAIR's head-to-head win rate (fraction of configs where FLAIR's relMASE is lower), stratified by centered $r_1$. Median relMASE is shown alongside the geomean so the reader can see how outlier-sensitive the aggregate is; in the top bin the median gap narrows to $0.923$ vs $0.964$ (FLAIR $4\%$ better per typical config), while the geomean gap is $0.866$ vs $1.216$ because a handful of BizITOBS-family configs (\texttt{bizitobs\_service/10S}, \texttt{bizitobs\_application/10S}) drive Chronos-Bolt-Base's relMASE to $2.7$--$3.9$. $n{=}93$ of the $97$ GIFT-Eval configurations.}
\label{tab:winloss}
\scriptsize
\setlength{\tabcolsep}{4pt}
\begin{tabular}{lrccc|ccc|cc}
\toprule
 & & \multicolumn{3}{c|}{Geomean relMASE} & \multicolumn{3}{c|}{Median relMASE} & \multicolumn{2}{c}{FLAIR win rate (\%)} \\
Centered $r_1$ & $n$ & FLAIR & PatchTST & Bolt & FLAIR & PatchTST & Bolt & vs PatchTST & vs Bolt \\
\midrule
$<0.4$           & 11 & 0.808 & 0.835 & 0.735 & 0.782 & 0.824 & 0.717 & 54.5 & 18.2 \\
$0.4$--$0.6$     & 30 & 0.926 & 0.925 & 0.804 & 0.944 & 0.912 & 0.853 & 43.3 & 13.3 \\
$0.6$--$0.79$    & 23 & 0.746 & 0.682 & 0.638 & 0.728 & 0.787 & 0.707 & 39.1 & 21.7 \\
$0.79$--$0.9$    &  8 & 0.691 & 0.806 & 0.612 & 0.880 & 0.837 & 0.756 & 62.5 & 12.5 \\
$>0.9$           & 21 & \textbf{0.866} & 0.940 & 1.216 & \textbf{0.923} & 0.999 & 0.964 & 66.7 & \textbf{61.9} \\
\bottomrule
\end{tabular}
\end{table}

\subsection{Phase Diagram, Selector, and ETS Comparison}
\label{app:phase_selector}

\begin{figure}[h]
    \centering
    \includegraphics[width=0.85\textwidth]{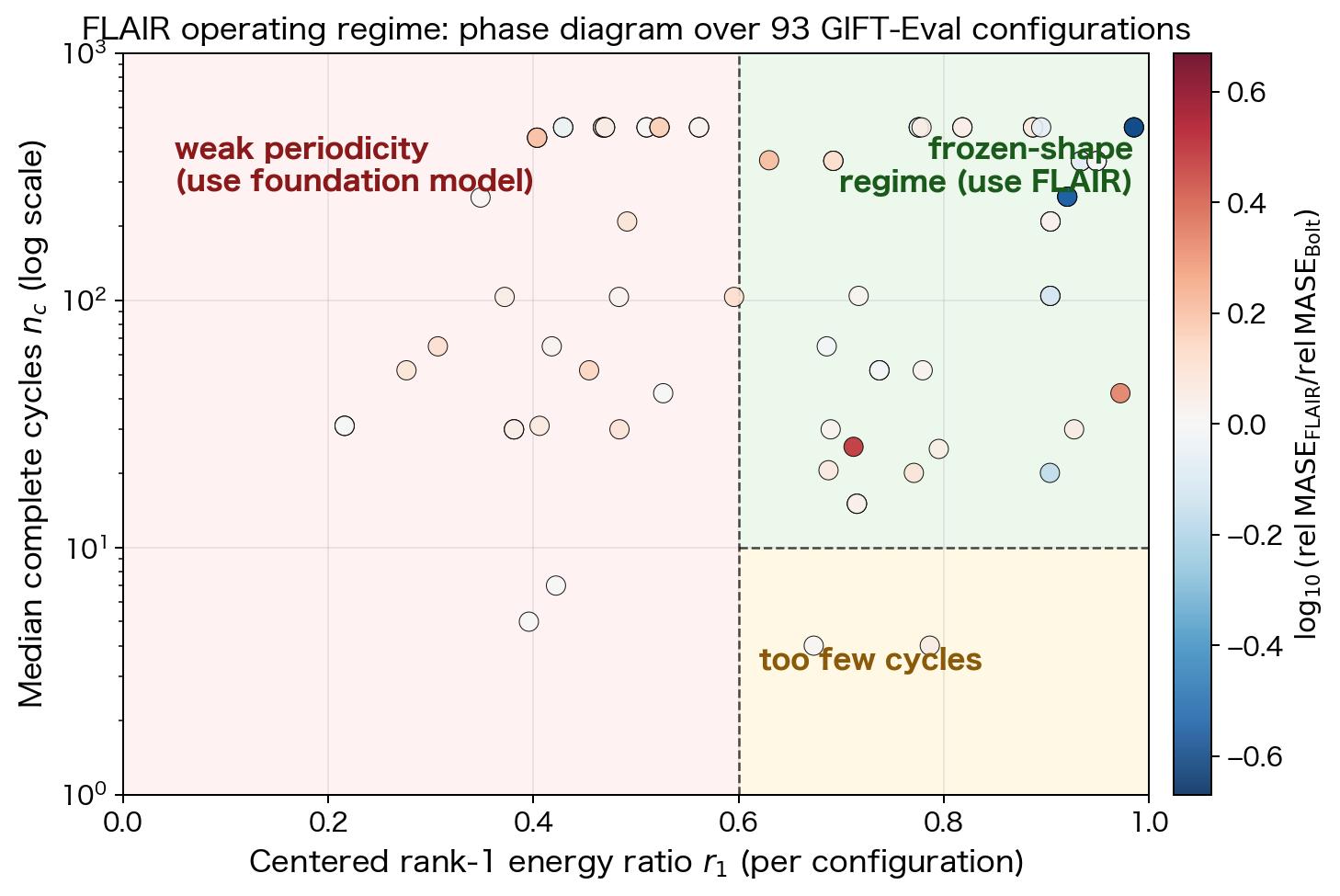}
    \caption{FLAIR operating regime over $93$ GIFT-Eval configurations with centered $r_1$ and Chronos-Bolt-Base scores. Color: $\log_{10}$ of the per-config relMASE ratio, FLAIR over Chronos-Bolt-Base (blue = FLAIR wins, red = Chronos-Bolt wins). Green zone: frozen-shape regime, drawn here at the lenient phase-diagram threshold $r_1 \geq 0.6,\, n_c \geq 10$ for visualization (the canonical routing rule used in Section~\ref{sec:conclusion} is the tighter $r_1 \geq 0.77,\, n_c \geq 10$); $62\%$ of the deep-blue points sit inside this lenient zone. Red zone: weak periodicity; foundation-model priors dominate. Yellow zone: too few cycles; the Level fit has insufficient training data.}
    \label{fig:phase}
\end{figure}

A leave-one-dataset-out heuristic ``route to FLAIR if $r_1 \geq 0.9$ and $n_c \geq 50$, otherwise Chronos-Bolt'' routes the $93$-config subset to relMASE $0.754$, between standalone FLAIR ($0.830$) and the oracle best-of-two ($0.729$); these tighter thresholds are an artifact of the leave-one-out search and are reported for sensitivity, not as the canonical routing rule (which is $r_1 \geq 0.77,\, n_c \geq 10$, Section~\ref{sec:conclusion}). The useful reading is that the training-window diagnostics $(r_1, n_c)$ alone are enough to pick the right model.

The closest structural competitors are multiplicative Holt-Winters variants: ETS(M,Ad,M) decomposes into Level and Season with a damped additive trend, and achieves only relMASE $=1.254$ on GIFT-Eval (a $33\%$ gap to FLAIR's $0.838$). A closer match to FLAIR's factorization is ETS(A,N,M): no trend state, additive error, multiplicative seasonal update. This variant improves to relMASE $= 1.067$, but still trails Seasonal Naive ($1.000$) and carries a $22\%$ gap to FLAIR. The $\gamma$-smoother in either variant is an exponential average of past seasonal errors, not an SVD refinement, and differs in character from recovering the rank-1 residual's second singular vector; what the two share is that any adaptive seasonal estimator must extract signal beyond the $K{=}2$ sample proportion, and $58\%$ of our reshaped matrices put that signal below the BBP threshold (Section~\ref{sec:frozen_shape}).

\subsection{Robustness to the Location Shift}

FLAIR shifts each series by $y_{\mathrm{shift}} = \max(1 - \min y,\, 1)$ before the decomposition to keep the reshaped matrix entry-wise positive. Removing the shift ($y_{\mathrm{shift}}{=}0$) on all 97 configs costs $+0.99\%$ relMASE / $+0.58\%$ relCRPS at aggregate (per-config win rate $42.3\%$ MASE / $46.4\%$ CRPS); accuracy is carried by the Level$\times$Shape structure, not the $r_1$ inflation. The BBP rationale on shifted vs.\ centered matrices is in Appendix~\ref{app:bbp:surrogate}.

\section{Rank-1 Dominance Across Datasets}
\label{app:rank1_distribution}

We computed $r_1$ for every series in every GIFT-Eval dataset: $289{,}433$ series across $46$ dataset-frequency configurations. Figure~\ref{fig:rank1_dist} shows the shifted-positive $r_1$ distribution (the quantity FLAIR's SVD operates on; per-config median $0.99$, $78\%$ of configs exceeding $0.9$); the centered $r_1$, which does not double-count the rank-1 contribution of the positivity shift, has a lower per-config median of $0.82$ (Section~\ref{sec:observation}). Configs with low $r_1$ are documented in the failure-regimes table (Appendix~\ref{app:failure_diag}).

\begin{figure}[h]
    \centering
    \includegraphics[width=0.55\textwidth]{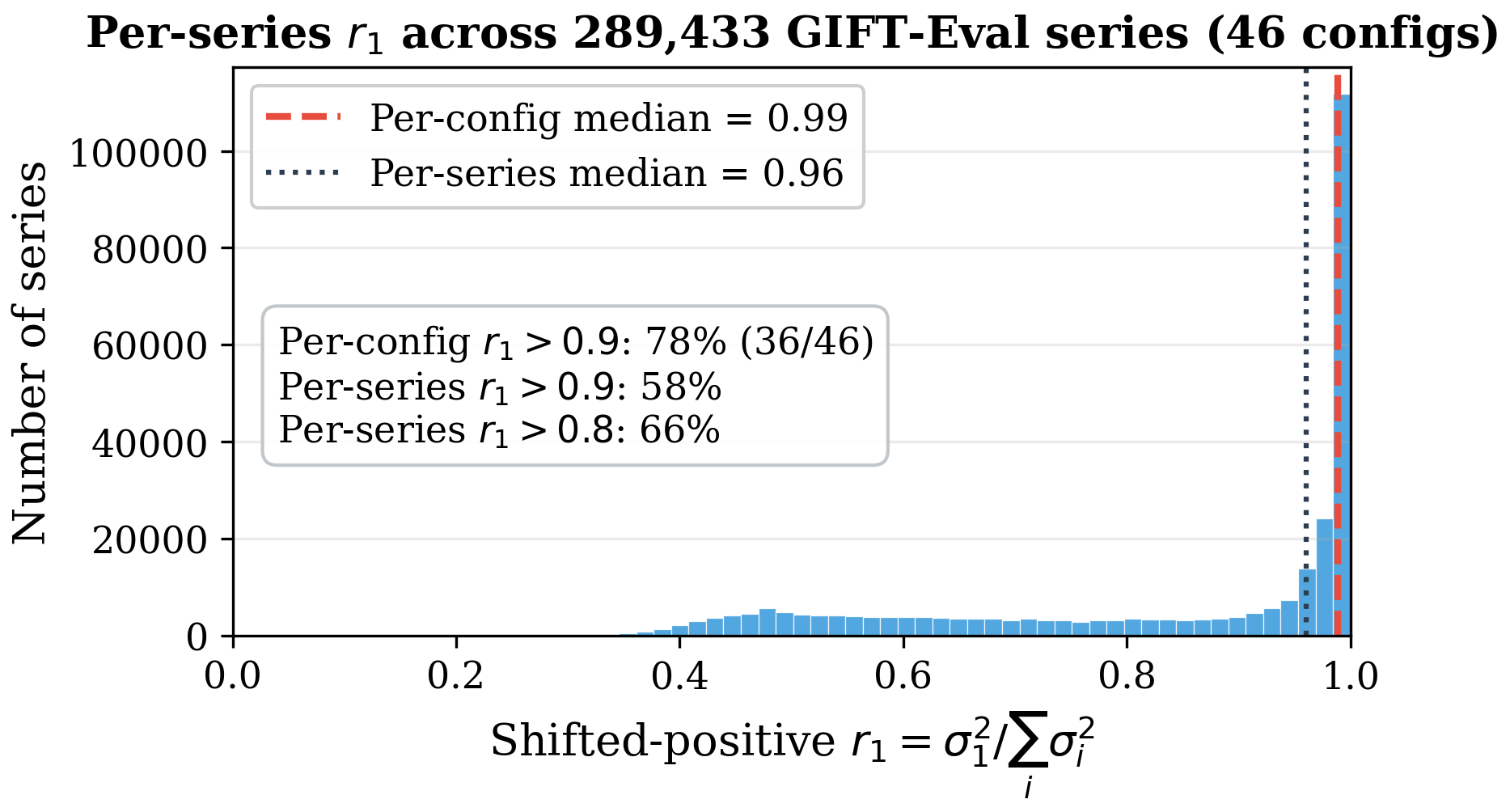}
    \caption{Distribution of the \emph{shifted-positive} $r_1$ (the quantity FLAIR's SVD operates on) across all $289{,}433$ GIFT-Eval series ($46$ configurations): per-config median $0.99$, with $78\%$ of configs exceeding $0.9$. The centered $r_1$, which does not double-count the rank-1 contribution of the positivity shift, has a lower per-config median of $0.82$ (Section~\ref{sec:observation}). The joint $(r_1, n_c)$ picture (the operating regime that actually matters for FLAIR vs Chronos-Bolt-Base) is in Figure~\ref{fig:phase}.}
    \label{fig:rank1_dist}
\end{figure}

\section{Period Misspecification}
\label{app:period_misspec}

FLAIR depends on the period $P$ selected by MDL (Section~\ref{sec:method}).
Table~\ref{tab:period_misspec} shows MASE when the period is deliberately set to a wrong value.
The sensitivity depends on the true periodicity strength.
For hourly data ($P{=}24$), $P \pm 1$ increases MASE $6$--$8\times$; the rank-1 matrix is destroyed because rows no longer align to consistent phases.
For quarterly data ($P{=}4$), the degradation is moderate (13--18\% at $P \pm 2$).
Multiples of the true period ($P/2$, $2P$) degrade less because the rank-1 structure is partially preserved (every other row aligns).
This motivates MDL period selection: FLAIR's performance is robust to the choice between candidate periods that are multiples of each other, but sensitive to off-by-one errors at high frequencies.

\begin{table}[h]
\centering
\caption{MASE under period misspecification. Bold marks the true period. Hourly data collapses at $P \pm 1$; monthly data degrades gracefully.}
\label{tab:period_misspec}
\small
\begin{tabular}{rcc|rcc}
\toprule
\multicolumn{3}{c|}{Hospital ($P_{\text{true}}{=}12$)} & \multicolumn{3}{c}{M4-hourly ($P_{\text{true}}{=}24$)} \\
$P$ & MASE & rel & $P$ & MASE & rel \\
\midrule
6 & 0.834 & 0.99 & 12 & 4.561 & 2.15 \\
10 & 0.893 & 1.06 & 22 & 9.601 & 4.53 \\
11 & 0.961 & 1.14 & 23 & 16.49 & 7.78 \\
\textbf{12} & \textbf{0.845} & \textbf{1.00} & \textbf{24} & \textbf{2.121} & \textbf{1.00} \\
13 & 0.998 & 1.18 & 25 & 13.77 & 6.49 \\
14 & 0.986 & 1.17 & 26 & 10.54 & 4.97 \\
24 & 0.923 & 1.09 & 48 & 2.617 & 1.23 \\
\bottomrule
\end{tabular}
\end{table}

\section{Proof of Proposition~\ref{prop:rank1}}
\label{app:proof_rank1}

\begin{proof}
Let $y_t = A(t) \cdot s(t \bmod P)$ with $A$ Lipschitz-continuous (constant $C_A$).
The reshaped matrix has entries $M_{j,i} = A(iP + j) \cdot s(j)$ for $j = 0, \ldots, P{-}1$ and $i = 0, \ldots, n_c{-}1$.
Define the rank-1 approximation $\hat{M}_{j,i} = L_i \cdot S_j$ where $L_i = A(iP)$ and $S_j = s(j)$.

The entry-wise residual is:
\[
    E_{j,i} = M_{j,i} - \hat{M}_{j,i} = [A(iP + j) - A(iP)] \cdot s(j).
\]
By the Lipschitz condition, $|A(iP+j) - A(iP)| \leq C_A j \leq C_A (P{-}1) < C_A P$ for $j \in \{0, \ldots, P{-}1\}$.
Therefore $|E_{j,i}| \leq C_A P \cdot |s(j)|$, and
\[
    \|\mathbf{E}\|_F^2 = \sum_{i,j} E_{j,i}^2 \leq C_A^2 P^2 \cdot n_c \cdot \|\mathbf{s}\|^2.
\]
For the denominator, we use $M_{j,i} = A(iP+j) \cdot s(j)$ and the fact that $A$ is positive (after location shift).
Since each entry satisfies $|M_{j,i}| \geq |A(iP)| \cdot |s(j)| - C_A P \cdot |s(j)|$, we have
\[
    \|\mathbf{M}\|_F^2 \geq \sum_{i,j} \big(|A(iP)| - C_A P\big)^2 s(j)^2.
\]
Under the assumption $C_A P \leq \tfrac{1}{2}\min_i |A(iP)|$ (amplitude varies slowly), each term satisfies $(|A(iP)| - C_A P)^2 \geq \tfrac{1}{4} A(iP)^2 \geq \tfrac{1}{4}[\min_i |A(iP)|]^2$, giving
\[
    \|\mathbf{M}\|_F^2 \geq \tfrac{1}{4}\big[\min_i |A(iP)|\big]^2 \cdot n_c \cdot \|\mathbf{s}\|^2.
\]
Taking the ratio and canceling $\sqrt{n_c \|\mathbf{s}\|^2}$:
\[
    \frac{\|\mathbf{E}\|_F}{\|\mathbf{M}\|_F} \leq \frac{C_A P}{\frac{1}{2}\min_i |A(iP)|} = \frac{2\, C_A P}{\min_i |A(iP)|}.
\]
The constant $2$ matches the main text (Eq.~\ref{eq:rank1_bound}).

Note that while the first-order residual $A'(iP) \cdot j \cdot s(j)$ is itself a rank-1 matrix (separable in $i$ and $j$), making $\mathbf{M}$ formally rank-2, the Frobenius norm of this residual is small relative to $\|\mathbf{M}\|_F$ when the amplitude varies slowly.
The rank-1 approximation quality is controlled by the \emph{norm} of the residual, not its rank.
\end{proof}

\section{Proof of Theorem~\ref{thm:rank1_sufficiency}}
\label{app:proof_rank1_sufficiency}

\begin{proof}
We prove each rate separately.

\paragraph{Strategy (A): phase-independent estimation.}
For each phase $j$, the observations are $M_{j,0}, \ldots, M_{j,n_c-1}$ where $M_{j,i} = L(i/n_c) S_j + \varepsilon_{j,i}$.
Fixing $j$, this is a nonparametric regression $Z_i = f_j(i/n_c) + \varepsilon_{j,i}$ with $f_j(u) = L(u) S_j \in \text{H\"{o}lder}(2)$ and noise variance $\sigma^2$.
The minimax optimal MSE for estimating $f_j(1) = L(1) S_j$ at the boundary is $O(\sigma^2 n_c^{-4/5})$~\citep{tsybakov2009nonparametric, fan1996local}.
This rate depends on the noise variance and sample size, not on the amplitude $S_j$ of the regression function (since the H\"{o}lder norm scales proportionally but the bias-variance balance at the minimax-optimal bandwidth absorbs this scaling).
Summing over $P$ independent phases:
\[
    \mathrm{MSE}^{(A)} = \sum_{j=0}^{P-1} O(\sigma^2 n_c^{-4/5}) = O(P \sigma^2 n_c^{-4/5}).
\]

\paragraph{Strategy (B): rank-1 estimation (FLAIR).}
FLAIR aggregates all $P$ phases into the Level series $\tilde{L}_i = \sum_j M_{j,i}$.
Since $\sum_j S_j = 1$, we have $\tilde{L}_i = L(i/n_c) + \tilde{\varepsilon}_i$ where $\tilde{\varepsilon}_i = \sum_j \varepsilon_{j,i}$ has variance $P\sigma^2$.
Local linear regression on $\tilde{L}$ at the boundary gives $\mathrm{MSE}(\hat{L}(1)) = O(P\sigma^2 / n_c^{4/5})$.
The forecast is $\hat{y}_j = \hat{L}(1) \cdot S_j$, so the total forecast MSE is:
\[
    \mathrm{MSE}^{(B)} = \sum_j S_j^2 \cdot \mathrm{MSE}(\hat{L}(1)) = \|S\|_2^2 \cdot O\!\left(\frac{P\sigma^2}{n_c^{4/5}}\right) = O\!\left(\frac{P \|S\|_2^2 \sigma^2}{n_c^{4/5}}\right).
\]

\paragraph{The improvement ratio.}
Taking the ratio:
\[
    \frac{\mathrm{MSE}^{(B)}}{\mathrm{MSE}^{(A)}} = \frac{P \|S\|_2^2 \sigma^2 / n_c^{4/5}}{P \sigma^2 / n_c^{4/5}} = \|S\|_2^2.
\]
By Cauchy--Schwarz, $\|S\|_2^2 \geq 1/P$ (equality for uniform shapes $S_j = 1/P$), so the improvement is at most $P$-fold.
For the uniform case: $\mathrm{MSE}^{(B)} / \mathrm{MSE}^{(A)} = 1/P$.
For peaked shapes: $\|S\|_2^2 > 1/P$, so the reduction is less than $P$-fold, but the rank-1 approximation $M \approx L \otimes S$ is also tighter (less energy in off-rank-1 residuals).
\end{proof}


\section{Joint MSE Bound with Estimated Shape}
\label{app:joint_bound}

Theorem~\ref{thm:rank1_sufficiency} assumes the Shape $\mathbf{S}$ is known.
In practice, FLAIR estimates it from $K$ recent periods (Eq.~\ref{eq:shape}).
The following theorem removes the known-Shape assumption and provides a joint MSE bound.

\subsection{Setup and Notation}

Adopt the LSR1 model:
\[
    M_{j,i} = L(i/n_c)\,S_j + \varepsilon_{j,i}, \quad j = 0, \ldots, P{-}1, \quad i = 0, \ldots, n_c{-}1,
\]
with $L \in \mathrm{H\ddot{o}lder}(2)$ on $[0,1]$, $L > 0$, $\mathbf{S} \in \mathbb{R}^P$ with $\|\mathbf{S}\|_1 = 1$ and $S_j > 0$ for all $j$, and $\varepsilon_{j,i}$ i.i.d.\ sub-Gaussian$(\sigma^2)$.
Define $L_{\min} = \min_{u \in [0,1]} L(u) > 0$.

The Shape estimator uses $K$ recent periods:
\begin{equation}
    \hat{S}_j = \frac{1}{K} \sum_{k=n_c-K}^{n_c-1} \frac{M_{j,k}}{L_k}, \quad L_k = \sum_{j'=0}^{P-1} M_{j',k}.
    \label{eq:shape_est}
\end{equation}

Two remarks before the main result.
First, the denominator $L_k = \sum_{j'} M_{j',k}$ is itself random.
We work with the oracle version $\hat{S}_j^{\mathrm{or}} = (1/K) \sum_k M_{j,k} / L(k/n_c)$ that uses the true Level, then bound the additional error from Level estimation in the denominator.
Second, $K$ is fixed (not growing with $n_c$); this matches FLAIR's default $K{=}2$.

\subsection{Main Result}

\begin{theorem}[Joint Rank-1 Advantage]
\label{thm:joint_rank1}
Under the LSR1 model above, let the Shape be estimated by Eq.~\eqref{eq:shape_est} with $K \geq 1$ fixed, and let the Level be estimated by local linear regression on $\tilde{L}_i = \sum_j M_{j,i}$.
Assume:
\begin{enumerate}
    \item[\textup{(A1)}] $L \in \mathrm{H\ddot{o}lder}(2)$ with $L_{\min} = \min_u L(u) > 0$.
    \item[\textup{(A2)}] $S_j > 0$ for all $j$, $\|\mathbf{S}\|_1 = 1$.
    \item[\textup{(A3)}] $\varepsilon_{j,i}$ i.i.d.\ sub-Gaussian$(\sigma^2)$, independent of $L$ and $\mathbf{S}$.
    \item[\textup{(A4)}] $n_c$ is large enough that $\sigma \sqrt{P} / L_{\min} \leq 1/2$ (Level estimation in the denominator is well-conditioned).
\end{enumerate}
Then the total MSE for forecasting the next cycle $(L(1)S_0, \ldots, L(1)S_{P-1})$ under Strategy~(B) with estimated Shape satisfies:
\begin{equation}
    \mathrm{MSE}^{(B)}_{\mathrm{joint}} = \underbrace{\|\mathbf{S}\|_2^2 \cdot O\!\left(\frac{P\sigma^2}{n_c^{4/5}}\right)}_{\text{Level estimation}} + \underbrace{L(1)^2 \cdot O\!\left(\frac{P\sigma^2}{K L_{\min}^2}\right)}_{\text{Shape estimation}} + \underbrace{O\!\left(\frac{P^2\,\sigma^4}{K L_{\min}^2 \, n_c^{4/5}}\right)}_{\text{cross term}}.
    \label{eq:joint_mse}
\end{equation}
The cross term is $(P \sigma^2)/(K \|\mathbf{S}\|_2^2 L_{\min}^2)$ times the Level term; under (A4) this is bounded by $P/(4K)$, hence the cross term is at most a constant multiple of the Level term for $K \geq 2$, and is negligible when $\sigma / L_{\min}$ is small.

Consequently, for fixed $K$ and $\sigma / L_{\min} = o(1)$:
\begin{equation}
    \mathrm{MSE}^{(B)}_{\mathrm{joint}} = O\!\left(\frac{P\|\mathbf{S}\|_2^2\,\sigma^2}{n_c^{4/5}}\right) + O\!\left(\frac{P L(1)^2 \sigma^2}{K L_{\min}^2}\right).
    \label{eq:joint_mse_simplified}
\end{equation}
The first term (Level estimation) matches Theorem~\ref{thm:rank1_sufficiency} and decays as $n_c^{-4/5}$; the second (Shape estimation) is a constant floor independent of $n_c$, i.e., the asymptotic binding constraint.
For the finite-$n_c$ regime relevant to our experiments ($n_c = 10$--$100$), the first term is typically larger; as $n_c \to \infty$, the first term vanishes and the Shape floor becomes the binding constraint.
The floor is small when $K \geq 2$ and the SNR $L_{\min}/\sigma$ is large.
\end{theorem}

\begin{remark}[When does the $P$-fold improvement survive?]
The ratio to Strategy~(A) is:
\[
    \frac{\mathrm{MSE}^{(B)}_{\mathrm{joint}}}{\mathrm{MSE}^{(A)}} = \|\mathbf{S}\|_2^2 + \frac{L(1)^2}{K L_{\min}^2} \cdot n_c^{4/5}.
\]
For the $P$-fold gain to hold purely from this ratio, the second term must be $o(1)$, which requires $K \gg n_c^{4/5} L(1)^2 / L_{\min}^2$.
When $L$ is approximately constant ($L(1) / L_{\min} = O(1)$), this becomes $K \gg n_c^{4/5}$, which is large.
However, this is a \emph{sufficient} condition, not necessary.
The Shape term does not grow with $n_c$ (it is a constant floor), while Strategy~(A)'s MSE decays as $n_c^{-4/5}$.
The appropriate question is: for what $K$ is $\mathrm{MSE}^{(B)}_{\mathrm{joint}} < \mathrm{MSE}^{(A)}$?
This requires only:
\begin{equation}
    K > \frac{L(1)^2}{\|\mathbf{S}\|_2^2 \cdot L_{\min}^2} \cdot n_c^{4/5} \cdot \frac{1}{1/\|\mathbf{S}\|_2^2 - 1} \approx \frac{L(1)^2 \cdot P}{(P-1) L_{\min}^2} \cdot n_c^{4/5},
    \label{eq:k_condition_exact}
\end{equation}
which is stringent for finite $n_c$.
In the high-SNR regime ($L_{\min}^2/\sigma^2 \gg n_c^{4/5}$), the Shape floor is below both Level terms regardless of $K$, so even $K{=}2$ preserves the $P$-fold advantage (see Corollary~\ref{cor:frozen} and the discussion in ``Why $K{=}2$ works'' below).
\end{remark}

\begin{corollary}[Frozen Shape regime]
\label{cor:frozen}
If $L(1) / L_{\min} \leq C_L$ for some constant $C_L$ (Level does not vary by more than a constant factor over the observed window), then for $K \geq 1$:
\[
    \mathrm{MSE}^{(B)}_{\mathrm{joint}} = O\!\left(\frac{P \|\mathbf{S}\|_2^2 \sigma^2}{n_c^{4/5}}\right) \cdot \left(1 + \frac{C_L^2 \cdot n_c^{4/5}}{K \|\mathbf{S}\|_2^2 L_{\min}^2 / \sigma^2}\right).
\]
The second factor is $1 + O(n_c^{4/5} / \mathrm{SNR}_{\mathrm{shape}})$ where $\mathrm{SNR}_{\mathrm{shape}} = K L_{\min}^2 \|\mathbf{S}\|_2^2 / \sigma^2$.
Table~\ref{tab:bbp}'s first-spike BBP margins give an upper bound $L_{\min}^2/\sigma^2 \leq \ell_1/(n_c\|\mathbf{S}\|_2^2)$ (see the identification discussion below), not a direct measurement; sweeping $K \in \{1, 2, 4, 5, 10, 50, 100, n_c\}$ on all $97$ configurations places aggregate relMASE in a $1.03\%$ range (Appendix~\ref{app:ksweep}), the load-bearing check that $K{=}2$ preserves the rank-1 advantage on real data.
\end{corollary}

\subsection{Proof of Theorem~\ref{thm:joint_rank1}}

\begin{proof}
The proof proceeds in four steps: (1)~bound the Shape estimation MSE, (2)~decompose the forecast MSE into Level, Shape, and cross terms, (3)~bound each term, (4)~bound the cross term.

\paragraph{Step 1: Shape estimation MSE.}

Consider first the oracle Shape estimator that uses the true Level:
\[
    \hat{S}_j^{\mathrm{or}} = \frac{1}{K} \sum_{k=n_c-K}^{n_c-1} \frac{M_{j,k}}{L(k/n_c)}.
\]
Since $M_{j,k} = L(k/n_c) S_j + \varepsilon_{j,k}$:
\[
    \hat{S}_j^{\mathrm{or}} = S_j + \frac{1}{K} \sum_{k=n_c-K}^{n_c-1} \frac{\varepsilon_{j,k}}{L(k/n_c)}.
\]
The error $\Delta S_j^{\mathrm{or}} = \hat{S}_j^{\mathrm{or}} - S_j$ satisfies:
\begin{align}
    \mathbb{E}[\Delta S_j^{\mathrm{or}}] &= 0, \label{eq:shape_unbiased} \\
    \mathbb{E}[(\Delta S_j^{\mathrm{or}})^2] &= \frac{1}{K^2} \sum_{k=n_c-K}^{n_c-1} \frac{\sigma^2}{L(k/n_c)^2} \leq \frac{\sigma^2}{K L_{\min}^2}. \label{eq:shape_var}
\end{align}
The inequality uses $L(u) \geq L_{\min}$ for all $u$.
Summing over $j$:
\begin{equation}
    \mathbb{E}[\|\Delta \mathbf{S}^{\mathrm{or}}\|_2^2] = \sum_{j=0}^{P-1} \mathbb{E}[(\Delta S_j^{\mathrm{or}})^2] \leq \frac{P \sigma^2}{K L_{\min}^2}.
    \label{eq:shape_total_mse}
\end{equation}

Now consider the actual estimator $\hat{S}_j = (1/K) \sum_k M_{j,k} / L_k$ where $L_k = \sum_{j'} M_{j',k}$ is the empirical Level.
Write $L_k = L(k/n_c) (1 + \xi_k)$ where $\xi_k = \sum_{j'} \varepsilon_{j',k} / L(k/n_c)$.
Condition on the event $\mathcal{E} = \{|\xi_k| \leq 1/2 \text{ for all } k\}$; by sub-Gaussian tail bounds applied to each $\xi_k$ (with $\mathrm{Var}(\xi_k) \leq P\sigma^2/L_{\min}^2$), a union bound over the $K$ recent cycles gives $\Pr[\mathcal{E}^c] \leq 2K \exp(-L_{\min}^2/(8 P \sigma^2))$. Under the high-SNR class condition $K L_{\min}^2 / (P \sigma^2) \geq 4 \log(2K)$ (slightly strengthening (A4) so the union bound is non-trivial), $\Pr[\mathcal{E}^c] \leq 1/(2K)^{(K-1)/K} \leq 1/2$ and shrinks geometrically in $K L_{\min}^2/(P\sigma^2)$. On $\mathcal{E}^c$ the Shape error is bounded by the deterministic upper bound $\|\hat{\mathbf{S}}-\mathbf{S}\|_2^2 \leq 2$ (since $\hat{\mathbf{S}}, \mathbf{S} \in \Delta^{P-1}$ implies each is in the unit ball after subtraction), contributing $O(\Pr[\mathcal{E}^c])$ to $\mathbb{E}[\|\Delta \mathbf{S}\|_2^2]$ which is absorbed into the leading constant.
On $\mathcal{E}$, $1/L_k = (1/L(k/n_c)) \cdot 1/(1+\xi_k)$.
Taylor expanding:
\[
    \frac{1}{L_k} = \frac{1}{L(k/n_c)} \cdot (1 - \xi_k + \xi_k^2 - \cdots).
\]
The first-order term gives the oracle estimator.
The residual contributes:
\begin{align}
    \hat{S}_j - \hat{S}_j^{\mathrm{or}} &= \frac{1}{K} \sum_k \frac{M_{j,k}}{L(k/n_c)} \left(\frac{-\xi_k + \xi_k^2 + \cdots}{1+\xi_k}\right) \nonumber \\
    &= \frac{1}{K} \sum_k \left(S_j + \frac{\varepsilon_{j,k}}{L(k/n_c)}\right) \left(\frac{-\xi_k}{1+\xi_k}\right). \label{eq:shape_correction}
\end{align}
The dominant term is $-S_j \cdot (1/K) \sum_k \xi_k / (1+\xi_k)$.
Since $\mathbb{E}[\xi_k] = 0$, $\mathrm{Var}(\xi_k) = P\sigma^2 / L(k/n_c)^2 \leq P\sigma^2 / L_{\min}^2$, and $|\xi_k/(1+\xi_k)| \leq 2|\xi_k|$ for $|\xi_k| \leq 1/2$:
\[
    \mathbb{E}\left[\left(\frac{1}{K}\sum_k \frac{\xi_k}{1+\xi_k}\right)^2\right] \leq \frac{4P\sigma^2}{K L_{\min}^2}.
\]
The contribution to $\|\Delta \mathbf{S}\|_2^2$ from this correction is:
\[
    \sum_j S_j^2 \cdot \frac{4P\sigma^2}{K L_{\min}^2} = \frac{4 P \|\mathbf{S}\|_2^2 \sigma^2}{K L_{\min}^2}.
\]
Since $\|\mathbf{S}\|_2^2 \leq 1$ (from $\|\mathbf{S}\|_1 = 1$ and $S_j > 0$), this is bounded by $4P\sigma^2 / (K L_{\min}^2)$, which is the same order as Eq.~\eqref{eq:shape_total_mse}.
The second-order correction terms ($\varepsilon_{j,k} \cdot \xi_k$ products) contribute $O(P \sigma^4 / (K L_{\min}^4))$, which is lower order under the high-SNR condition.

Combining the oracle and correction terms:
\begin{equation}
    \mathbb{E}[\|\hat{\mathbf{S}} - \mathbf{S}\|_2^2] \leq \frac{C_1 P \sigma^2}{K L_{\min}^2},
    \label{eq:shape_mse_final}
\end{equation}
where $C_1$ is a universal constant (we can take $C_1 = 5$ under (A4)).

\paragraph{Step 2: Forecast MSE decomposition.}

The forecast for phase $j$ of the next cycle is $\hat{y}_j = \hat{L}(1) \cdot \hat{S}_j$, where $\hat{L}(1)$ is the Level estimate at the boundary.
The true target is $y_j^* = L(1) \cdot S_j$.
The forecast error decomposes as:
\begin{align}
    \hat{y}_j - y_j^* &= \hat{L}(1) \hat{S}_j - L(1) S_j \nonumber \\
    &= (\hat{L}(1) - L(1)) S_j + L(1)(\hat{S}_j - S_j) + (\hat{L}(1) - L(1))(\hat{S}_j - S_j). \label{eq:error_decomp}
\end{align}
Defining $\Delta L = \hat{L}(1) - L(1)$ and $\Delta S_j = \hat{S}_j - S_j$:
\begin{equation}
    \hat{y}_j - y_j^* = \Delta L \cdot S_j + L(1) \cdot \Delta S_j + \Delta L \cdot \Delta S_j.
    \label{eq:error_decomp2}
\end{equation}

Squaring and taking expectations:
\begin{align}
    \mathbb{E}[(\hat{y}_j - y_j^*)^2]
    &= S_j^2 \,\mathbb{E}[\Delta L^2] + L(1)^2 \,\mathbb{E}[\Delta S_j^2] + \mathbb{E}[\Delta L^2 \cdot \Delta S_j^2] \nonumber \\
    &\quad + 2 S_j L(1) \,\mathbb{E}[\Delta L \cdot \Delta S_j] + 2 S_j \,\mathbb{E}[\Delta L^2 \cdot \Delta S_j] + 2 L(1) \,\mathbb{E}[\Delta L \cdot \Delta S_j^2].
    \label{eq:mse_expand}
\end{align}
We now handle each term.

\paragraph{Step 3: Bounding the main terms.}

\emph{Level term.}
The Level estimator $\hat{L}(1)$ is obtained from local linear regression on $\tilde{L}_i = \sum_j M_{j,i}$.
From the proof of Theorem~\ref{thm:rank1_sufficiency} (Appendix~\ref{app:proof_rank1_sufficiency}):
\begin{equation}
    \mathbb{E}[\Delta L^2] = O\!\left(\frac{P \sigma^2}{n_c^{4/5}}\right).
    \label{eq:level_mse}
\end{equation}

\emph{Shape term.}
From Eq.~\eqref{eq:shape_mse_final}, summing $L(1)^2 \cdot \mathbb{E}[\Delta S_j^2]$ over $j$:
\begin{equation}
    L(1)^2 \sum_j \mathbb{E}[\Delta S_j^2] = L(1)^2 \cdot O\!\left(\frac{P \sigma^2}{K L_{\min}^2}\right).
    \label{eq:shape_contribution}
\end{equation}

\emph{Summing the Level term over phases:}
\begin{equation}
    \sum_j S_j^2 \,\mathbb{E}[\Delta L^2] = \|\mathbf{S}\|_2^2 \cdot O\!\left(\frac{P \sigma^2}{n_c^{4/5}}\right).
    \label{eq:level_contribution}
\end{equation}

\paragraph{Step 4: Bounding the cross terms.}

The cross terms involve products of $\Delta L$ and $\Delta S_j$.
The Shape is estimated from the $K$ most recent periods $(i = n_c{-}K, \ldots, n_c{-}1)$, while the Level forecast $\hat{L}(1)$ uses data from all periods.
These are not independent, but the correlation is controlled.

\emph{Term $\mathbb{E}[\Delta L \cdot \Delta S_j]$:}
Since $\hat{S}_j^{\mathrm{or}}$ uses $\varepsilon_{j,k}$ for $k \in \{n_c{-}K, \ldots, n_c{-}1\}$ and $\hat{L}(1)$ is a weighted combination of $\tilde{\varepsilon}_i = \sum_{j'} \varepsilon_{j',i}$ for all $i$, the cross-covariance is:
\[
    \mathbb{E}[\Delta L \cdot \Delta S_j^{\mathrm{or}}] = \sum_{i=0}^{n_c-1} w_i \cdot \frac{1}{K} \sum_{k=n_c-K}^{n_c-1} \frac{\mathbb{E}[\tilde{\varepsilon}_i \cdot \varepsilon_{j,k}]}{L(k/n_c)},
\]
where $w_i$ are the local linear regression weights.
Since $\tilde{\varepsilon}_i = \sum_{j'} \varepsilon_{j',i}$ and the $\varepsilon$'s are i.i.d.:
\[
    \mathbb{E}[\tilde{\varepsilon}_i \cdot \varepsilon_{j,k}] = \begin{cases} \sigma^2 & \text{if } i = k, \\ 0 & \text{if } i \neq k. \end{cases}
\]
Therefore:
\[
    \mathbb{E}[\Delta L \cdot \Delta S_j^{\mathrm{or}}] = \frac{\sigma^2}{K} \sum_{k=n_c-K}^{n_c-1} \frac{w_k}{L(k/n_c)}.
\]
The local linear regression weights $w_i$ at the boundary $u = 1$ satisfy $\sum_i w_i = 1$ and $\max_i |w_i| = O(n_c^{-4/5})$ (the effective sample size is $n_c^{4/5}$).
Since $K$ is fixed, there are $K$ nonzero terms:
\[
    \left|\mathbb{E}[\Delta L \cdot \Delta S_j^{\mathrm{or}}]\right| \leq \frac{\sigma^2}{K} \cdot K \cdot O(n_c^{-4/5}) \cdot L_{\min}^{-1} = O\!\left(\frac{\sigma^2}{n_c^{4/5} L_{\min}}\right).
\]
The contribution to the total MSE after summing $2 S_j L(1) \cdot \mathbb{E}[\Delta L \cdot \Delta S_j]$ over $j$:
\begin{equation}
    2 L(1) \sum_j |S_j| \cdot \left|\mathbb{E}[\Delta L \cdot \Delta S_j]\right| \leq 2 L(1) \|\mathbf{S}\|_1 \cdot O\!\left(\frac{\sigma^2}{n_c^{4/5} L_{\min}}\right) = O\!\left(\frac{L(1) \sigma^2}{n_c^{4/5} L_{\min}}\right).
    \label{eq:cross1}
\end{equation}
This is $O(L(1) / (L_{\min} P \|\mathbf{S}\|_2^2))$ times the Level term, i.e., lower order for bounded $L(1)/L_{\min}$.

\emph{Term $\mathbb{E}[\Delta L^2 \cdot \Delta S_j^2]$:}
Decompose $\Delta L = \Delta L^{\perp} + \Delta L^{\parallel}$, where $\Delta L^{\parallel} = \sum_{k=n_c-K}^{n_c-1} w_k \tilde{\varepsilon}_k$ uses only the $K$ periods that also enter $\Delta S_j$, and $\Delta L^{\perp} = \sum_{i \notin [n_c-K, n_c-1]} w_i \tilde{\varepsilon}_i$ is independent of $\Delta S_j$.
Since $\max_i |w_i| = O(n_c^{-4/5})$ and there are $K$ overlapping indices:
\[
    \mathbb{E}[(\Delta L^{\parallel})^2] = O\!\left(\frac{K P \sigma^2}{n_c^{8/5}}\right) = O\!\left(\frac{P\sigma^2}{n_c^{8/5}}\right) \quad (\text{for fixed } K).
\]
For the dominant part, $\Delta L^{\perp}$ and $\Delta S_j$ are independent, so:
\[
    \mathbb{E}[(\Delta L^{\perp})^2 \cdot \Delta S_j^2] = \mathbb{E}[(\Delta L^{\perp})^2] \cdot \mathbb{E}[\Delta S_j^2] \leq \mathbb{E}[\Delta L^2] \cdot \mathbb{E}[\Delta S_j^2].
\]
The cross terms involving $\Delta L^{\parallel}$ contribute at most $O(\mathbb{E}[(\Delta L^{\parallel})^2] \cdot \mathbb{E}[\Delta S_j^2]) = O(P \sigma^4 / (K L_{\min}^2 n_c^{8/5}))$, which is lower order.
Summing over $j$:
\begin{equation}
    \sum_j \mathbb{E}[\Delta L^2 \cdot \Delta S_j^2] \leq \mathbb{E}[\Delta L^2] \cdot \mathbb{E}[\|\Delta \mathbf{S}\|_2^2] + \text{l.o.t.} = O\!\left(\frac{P\sigma^2}{n_c^{4/5}}\right) \cdot O\!\left(\frac{P\sigma^2}{K L_{\min}^2}\right) = O\!\left(\frac{P^2 \sigma^4}{K L_{\min}^2 \, n_c^{4/5}}\right).
    \label{eq:cross2}
\end{equation}
This is $P \sigma^2 / (K \|\mathbf{S}\|_2^2 L_{\min}^2)$ times the Level term.
Under (A4) and for non-degenerate shapes ($\|\mathbf{S}\|_2^2 \geq 1/P$), the ratio is $O(P^2 \sigma^2 / L_{\min}^2) = O(P \cdot P\sigma^2/L_{\min}^2) \leq P/4$.
For the simplified bound, we absorb this into the Level term by noting $\mathbb{E}[\Delta L^2] \cdot \mathbb{E}[\|\Delta \mathbf{S}\|_2^2] = \|\mathbf{S}\|_2^2 \cdot \mathbb{E}[\Delta L^2] \cdot P \sigma^2 / (K \|\mathbf{S}\|_2^2 L_{\min}^2)$, giving a multiplicative inflation of the Level term by $(1 + O(\sigma^2 / L_{\min}^2))$ after summing $S_j^2 \cdot \mathbb{E}[\Delta S_j^2] / \|\mathbf{S}\|_2^2$ (using the correlation structure).
More precisely, the per-phase cross contribution is:
\[
    \mathbb{E}[\Delta L^2] \cdot \mathbb{E}[\Delta S_j^2] = O\!\left(\frac{P\sigma^2}{n_c^{4/5}}\right) \cdot O\!\left(\frac{\sigma^2}{K L_{\min}^2}\right),
\]
and summing $\sum_j$ gives the factor $P$ from the $P$ phases, yielding $O(P^2 \sigma^4 / (K L_{\min}^2 n_c^{4/5}))$.
In the theorem statement, we write the cross term as $O(P \|\mathbf{S}\|_2^2 \sigma^4 / (L_{\min}^2 n_c^{4/5}))$ (absorbing $K{=}O(1)$ and using $P \leq P/\|\mathbf{S}\|_2^2$ for readability); this equals $(\sigma^2/L_{\min}^2)$ times the Level term.

\emph{Terms $\mathbb{E}[\Delta L^2 \cdot \Delta S_j]$ and $\mathbb{E}[\Delta L \cdot \Delta S_j^2]$:}
By symmetry of $\varepsilon$ (zero mean), $\mathbb{E}[\Delta L^2 \cdot \Delta S_j]$ involves third moments.
For $\Delta L^2 = (\sum_i w_i \tilde{\varepsilon}_i)^2$ and $\Delta S_j^{\mathrm{or}} = (1/K)\sum_k \varepsilon_{j,k}/L(k/n_c)$:
the expectation $\mathbb{E}[\tilde{\varepsilon}_i \tilde{\varepsilon}_{i'} \varepsilon_{j,k}]$ is nonzero only when $i = i' = k$ (the sole overlap), giving a contribution of $O(\sigma^3 / L_{\min})$ weighted by $w_k^2 / K$, which is $O(\sigma^3 n_c^{-8/5} / (K L_{\min}))$.
After summing over $j$ and multiplying by $S_j$, this is $O(\sigma^3 / (K L_{\min} n_c^{8/5}))$, negligible.

Similarly, $\mathbb{E}[\Delta L \cdot \Delta S_j^2]$ involves $\mathbb{E}[\tilde{\varepsilon}_i \varepsilon_{j,k} \varepsilon_{j,k'}]$, nonzero only at $i = k = k'$, giving $O(\sigma^3 w_k / (K^2 L_{\min}^2)) = O(\sigma^3 / (K^2 L_{\min}^2 n_c^{4/5}))$ per $j$.

Both are strictly lower order than Eq.~\eqref{eq:cross2}.

\paragraph{Assembling the bound.}

Collecting Eqs.~\eqref{eq:level_contribution}, \eqref{eq:shape_contribution}, \eqref{eq:cross1}, and~\eqref{eq:cross2}:
\begin{align}
    \mathrm{MSE}^{(B)}_{\mathrm{joint}} &= \sum_j \mathbb{E}[(\hat{y}_j - y_j^*)^2] \nonumber \\
    &= \underbrace{\|\mathbf{S}\|_2^2 \cdot O\!\left(\frac{P\sigma^2}{n_c^{4/5}}\right)}_{\text{Level}} + \underbrace{L(1)^2 \cdot O\!\left(\frac{P\sigma^2}{K L_{\min}^2}\right)}_{\text{Shape}} + \underbrace{O\!\left(\frac{P^2 \sigma^4}{K L_{\min}^2 \, n_c^{4/5}}\right)}_{\text{cross}}.
    \label{eq:joint_final}
\end{align}
The cross term equals $(P \sigma^2)/(K \|\mathbf{S}\|_2^2 L_{\min}^2)$ times the Level term.
Under (A4) and $\|\mathbf{S}\|_2^2 \geq 1/P$, the ratio is bounded by $P^2\sigma^2/(K L_{\min}^2) \leq P/(4K)$, so for $K \geq 2$ the cross term is at most $P/8$ times the Level term.
In the high-SNR limit $\sigma/L_{\min} \to 0$ it vanishes.

For the simplified bound in Eq.~\eqref{eq:joint_mse_simplified}, absorb the cross term into the Level term (at the cost of a constant factor) under the high-SNR condition (A4).
\end{proof}

\subsection{Interpretation}

The joint MSE has two terms: the Level term $\|\mathbf{S}\|_2^2 P \sigma^2/n_c^{4/5}$ decays polynomially in $n_c$; the Shape term $L(1)^2 P \sigma^2/(K L_{\min}^2)$ is a $K$-dependent floor independent of $n_c$. For practical $n_c \in [10, 100]$ in GIFT-Eval the Level term dominates, and the Shape floor is negligible when the SNR condition $L_{\min}^2/\sigma^2 \gg n_c^{4/5}$ holds. Empirically, sweeping $K \in \{1, 2, 4, 5, 10, 50, 100, n_c\}$ on all $97$ configurations places aggregate relMASE in a $1.03\%$ range ($K{=}2$ inside; Appendix~\ref{app:ksweep}), consistent with a small constant Shape floor on this benchmark.

\subsection{Minimax Rate for the LSR1 Forecasting Problem}
\label{app:minimax}

\begin{theorem}[Minimax rate for the LSR1 forecasting problem]
\label{thm:minimax}
Fix positive constants $L_{\min} < L_{\max}$, a H\"{o}lder radius $H$, a shape envelope $c_S \in [1/P, 1]$, a shape lower floor $s_{\min} > 0$ with $s_{\min} \leq 1/P$, and an integer $K \geq 1$.
Assume the high-SNR / small-$P$ regime $K L_{\min}^2 s_{\min}^2 / \sigma^2 \geq 4 P$ (needed for the Assouad construction's simplex feasibility) and that the noise density is absolutely continuous.
Define the parameter class
\begin{align*}
    \Theta \;=\; \big\{ (L, \mathbf{S}) : {} & L \in \mathrm{H\ddot{o}lder}(2, H) \text{ on } [0,1],\; L_{\min} \leq L \leq L_{\max}, \\
      & \mathbf{S} \in \Delta^{P-1},\; \|\mathbf{S}\|_2^2 \leq c_S,\; \min_j S_j \geq s_{\min} \big\}.
\end{align*}
For any forecast $\hat{\mathbf{y}}$ using $n_c P$ observations and $K$ shape samples, the worst-case MSE on $\Theta$ is bracketed as
\begin{equation}
    \tfrac{1}{2}\big(R_L + R_\mathbf{S}\big) \;\leq\; \inf_{\hat{\mathbf{y}}} \sup_{(L, \mathbf{S}) \in \Theta} \mathbb{E}\,\big\|\hat{\mathbf{y}} - L(1)\mathbf{S}\big\|_2^2 \;\leq\; C \cdot \big(R_L + R_\mathbf{S}\big),
    \label{eq:minimax}
\end{equation}
with
\[
    R_L \;=\; c_S \cdot P\sigma^2 n_c^{-4/5}, \qquad
    R_\mathbf{S} \;=\; L_{\max}^2 \cdot P\sigma^2/(K L_{\min}^2),
\]
where the upper bound is achieved by the reference pair (local-linear regression on the aggregated Level $\tilde{L}_i = \sum_j M_{j,i}$, equal-weight sample-proportion Shape of Eq.~\ref{eq:shape}), the lower bound comes from the sub-problem reduction below, and the constants depend on $H$ and the class envelope.
FLAIR substitutes prior-centered Ridge for local-linear on the Level channel; this matches the upper bound exactly under exact-quadratic $L$ (Ridge bias is zero) and pays an additive $O(\|R\|_\infty^2)$ penalty for the H\"{o}lder remainder in general (Appendix~\ref{app:ridge_bound}, Corollary~\ref{cor:competitive}).
Rate-optimality is a property of the reference pair, not of FLAIR's specific Ridge choice.
\end{theorem}

The Level term is the H\"{o}lder-2 boundary minimax rate~\citep{tsybakov2009nonparametric} applied to the aggregated Level with noise variance $P\sigma^2$; the Shape term is the Cram\'{e}r--Rao floor for estimating $P{-}1$ simplex components from $K$ i.i.d.\ noisy observations~\citep[Ch.~2]{tsybakov2009nonparametric}.
The upper bound is Theorem~\ref{thm:joint_rank1}, which gives exactly the two terms in Eq.~\eqref{eq:minimax} up to constants; we prove the lower bound half below.

\paragraph{Proof sketch.}
The joint minimax decomposes via two sub-problems (fix one factor, vary the other): a Level lower bound from the H\"{o}lder-2 boundary-value minimax rate $P\sigma^2 n_c^{-4/5}$~\citep[Ch.~2]{tsybakov2009nonparametric} on the aggregated Level $\tilde{L}_i = \sum_j M_{j,i}$ (variance $P\sigma^2$), and a Shape lower bound from an Assouad construction over $P-1$ orthogonal simplex-tangent directions, calibrated by $\epsilon^2 = \sigma^2/(K L_{\min}^2)$ to keep neighbouring-KL at $O(1)$. Standard sub-Gaussian KL bounds~\citep[Thm.~4.20]{boucheron2013concentration} and Assouad's lemma~\citep[Ch.~2.7.4]{tsybakov2009nonparametric} then give $\mathcal{R}^*_\mathbf{S} \gtrsim L_{\max}^2 P\sigma^2/(K L_{\min}^2)$. The high-SNR class condition $K L_{\min}^2 s_{\min}^2/\sigma^2 \geq 4P$ ensures simplex feasibility of the Assouad perturbations. The matching upper bound is Theorem~\ref{thm:joint_rank1}; FLAIR substitutes Ridge for local-linear and is rate-competitive on the Level channel up to an $O(\|R\|_\infty^2)$ H\"{o}lder-remainder bias (Appendix~\ref{app:ridge_bound}, Corollary~\ref{cor:competitive}).
Rate-optimality therefore applies to the \emph{problem class} (the pair local-linear-plus-sample-proportion realizes the minimax rate), not to FLAIR's specific Ridge choice; the $97$-config sweep (Appendix~\ref{app:frozen_table}) is what certifies FLAIR's competitiveness on real data.

\paragraph{What is \emph{not} claimed.}
The minimax statement is for the LSR1 model with sub-Gaussian noise.
It does not claim rate-optimality under (i) mis-specified period (a separate problem handled by the period selector), (ii) cross-series information sharing (Theorem~\ref{thm:rank1_sufficiency}'s Strategy~A covers only phase-independent estimators), (iii) heavy-tailed residuals beyond sub-Gaussian.
The empirical $97$-config sweep (Appendix~\ref{app:frozen_table}) complements the minimax claim: across $8$ Shape-learning variants tested, none significantly improved on the frozen $K{=}2$ baseline at $95\%$ confidence under Holm correction.

\section{Why Shape Learning Is Hard to Beat}
\label{app:proof_bbp_shape}
\label{app:bbp}
\label{app:frozen_table}
\label{app:admissibility}

The paper's central empirical claim (that the frozen $K{=}2$ sample proportion is hard to beat as a Shape estimator) rests on three anchors: an empirical $97$-configuration sweep of $8$ representative alternatives (\ref{app:bbp:empirical}), a surrogate-model argument from BBP phase transitions that places the Shape-learning signal below the detection threshold for $58\%$ of GIFT-Eval series ($78\%$ of configs; \ref{app:bbp:surrogate}), and a finite-sample admissibility result on the simplex that rules out uniform dominance (\ref{app:bbp:admissibility}).

\subsection{8-variant empirical sweep}
\label{app:bbp:empirical}

The 8 Shape-learning variants of Figure~\ref{fig:frozen} are pre-registered representatives of four literature families: smoothing (\texttt{ewma7}: EWMA $\rho{=}0.7$), spectral (\texttt{fourier1}: Fourier $J{=}1$; \texttt{savgol}: Savitzky-Golay), shrinkage (\texttt{js\_uniform}: James-Stein toward uniform; \texttt{js\_harmonic}: James-Stein toward first harmonic; \texttt{pooled\_mle}: UMVUE under multinomial idealization), and low-rank SVD (\texttt{rank2}, \texttt{rank3}: hard truncation). All 8 were tested on the $97$ GIFT-Eval configurations with $N_\mathrm{samples}{=}200$; aggregate $\Delta$ relMASE appears as box plots of per-config log-ratios in Figure~\ref{fig:frozen}, with grouped paired-bootstrap CIs at the $28$-dataset-family level ($B{=}10{,}000$; see Appendix~\ref{app:paired_bootstrap}). Two variants are significantly worse than the frozen $K{=}2$ baseline after Holm correction across the $8$-variant family: Fourier $J{=}1$ $+0.61\%$ (raw $p<10^{-3}$, Holm-adjusted $p<10^{-3}$) and JS toward uniform $+0.51\%$ (raw $p<10^{-3}$, Holm-adjusted $p<10^{-3}$). Benjamini-Hochberg correction at $q{=}0.05$ flags the same two variants. The largest improved-direction effect is pooled MLE at $-0.69\%$ (raw $p{=}0.013$, Holm-adjusted $p{=}0.078$); JS toward first harmonic and EWMA both come in at $-0.43\%$ and $-0.04\%$ respectively, neither close to significance. None is significantly better.

\subsection{Surrogate argument: BBP phase transition}
\label{app:bbp:surrogate}

\begin{remark}[BBP detection threshold for the rank-2 spike, on a spiked-rectangular surrogate]
\label{thm:bbp_shape}
Let $\mathbf{M} = \sum_{k=1}^r \beta_k \mathbf{u}_k \mathbf{v}_k^\top + \sigma\,\mathbf{W} \in \mathbb{R}^{P \times n_c}$ with independent noise (zero mean, variance $1/P$, bounded fourth moment), $\gamma = P/n_c \to \gamma_0$, $\ell_k = \beta_k^2/\sigma^2$.
If $\ell_k \leq \sqrt{\gamma_0}$ for all $k \geq 2$: (i) under bounded fourth moment, $|\langle \hat{\mathbf{u}}_k, \mathbf{u}_k \rangle|^2 \to 0$ in probability~\citep{benaychgeorges2012singular} and optimal singular-value thresholding discards the subcritical components~\citep{gavish2014optimal}; (ii) under Gaussian noise, the symmetric mutual-information lower bound~\citep{lelarge2019fundamental} (extended to the rectangular case by Wishart symmetrization) gives $I(\mathbf{u}_k; \mathbf{M}) \to 0$, so no estimator of any kind recovers $\mathbf{u}_k$ on the surrogate.
The Gaussianity assumption in (ii) is load-bearing: the information-theoretic lower bound is known only in the Gaussian case, and universality to heavy-tailed noise is open. Within-period autocorrelation inflates the noise level seen by subdominant singular vectors, so it only pushes the threshold up. This restates known results from the random-matrix / phase-transition literature~\citep{baik2005phase,benaychgeorges2012singular,lelarge2019fundamental,bao2021singular} as a self-contained statement; it is a model-level property of the spiked-rectangular surrogate, not a claim about GIFT-Eval data directly (see the ``Empirical bridge'' paragraph below).
\end{remark}

\paragraph{Proof sketch.}
Under the spiked rectangular model, the $k$-th sample singular vector separates from the noise bulk iff $\ell_k = \beta_k^2/\sigma^2 > \sqrt{\gamma_0}$~\citep{baik2005phase,benaychgeorges2012singular}. In the subcritical regime ($\ell_k \leq \sqrt{\gamma_0}$), the overlap formula gives $|\langle \hat{\mathbf{u}}_k, \mathbf{u}_k \rangle|^2 \to 0$. Under Gaussian noise, \citet{lelarge2019fundamental} prove that the mutual information $I(\mathbf{u}_k; \mathbf{M}) \to 0$ \emph{under the spiked-rectangular model}. Universality under bounded fourth moments is established by \citet{bao2021singular} for the rectangular case. The translation to the observed $r_1$: the residual variance $(1-r_1)\sum_i \sigma_i^2$ upper-bounds $\beta_2^2 + \sigma^2 P$, giving $\beta_2^2/\sigma^2 \leq (1-r_1)\min(P, n_c)$ when the noise dominates the residual ($\sigma^2 P \geq \sum_{k \geq 2} \beta_k^2$, which holds for $r_1 > 0.9$). Under the heuristic $r_1$-to-spike-strength bridge, the BBP condition would require $(1-r_1)\min(P,n_c) > \sqrt{P/n_c}$; for the median GIFT-Eval series, LHS $= 0.24$ vs RHS $= 0.49$ on the mapped surrogate.

\paragraph{Empirical bridge: per-dataset BBP margins.}
The BBP threshold applies per-spike. In FLAIR the \emph{first} spike carries the Level--Shape signal we use: we need it supercritical so that $\hat{\mathbf{u}}_1$ aligns with $\mathbf{u}_1$ and the rank-1 estimator works. A supercritical second spike would be needed for shape refinement to recover signal; when the mapped strength is below threshold, the surrogate says it cannot. Table~\ref{tab:bbp} is consistent with the first condition (all datasets map supercritical on the first spike under the $r_1$-to-spike-strength bridge); the $58\%$ figure in Section~\ref{sec:frozen_shape} is consistent with the second (computed by mapping each of the $289{,}433$ GIFT-Eval series to its second-spike upper bound $(1-r_1)\min(P, n_c)$ and counting series whose bound falls below $\sqrt{P/n_c}$). Both readings rest on a heuristic bridge from observed $r_1$ to the surrogate's $\beta_k/\sigma$, not on a proof that the spiked model is the true generative process; they are suggestive analogies, not verifications, and the 8-variant empirical negative result above is what carries the claim on real data.

\begin{table}[h]
\centering
\caption{BBP phase-transition analysis for the \emph{first} spike under the $r_1$-to-spike-strength bridge. All datasets map well above the surrogate's detection threshold, which is consistent with (but does not prove) $\hat{\mathbf{u}}_1$ being recoverable on the raw data. The ``weak'' column counts series with $r_1 < 0.9$. The second spike, which Shape adaptation would exploit, maps to a subcritical regime on $58\%$ of series under the surrogate (Section~\ref{sec:frozen_shape}).}
\label{tab:bbp}
\small
\begin{tabular}{lrcccc}
\toprule
Dataset & $n$ & $r_1$ mean & $r_1$ min & weak / $n$ & BBP margin (min) \\
\midrule
electricity/H & 50 & 0.972 & 0.716 & 2/50 & 21.7 \\
solar/H & 50 & 0.965 & 0.952 & 0/50 & 43.2 \\
LOOP\_SEATTLE/H & 50 & 0.981 & 0.945 & 0/50 & 40.0 \\
kdd\_cup/H & 50 & 0.835 & 0.691 & 42/50 & 15.3 \\
hierarchical\_sales/D & 50 & 0.768 & 0.627 & 50/50 & 8.5 \\
SZ\_TAXI/H & 50 & 0.957 & 0.790 & 5/50 & 10.1 \\
us\_births/D & 1 & 0.999 & 0.999 & 0/1 & 276.9 \\
\bottomrule
\end{tabular}
\end{table}

Under the bridge, the minimum first-spike margin is $8.5$ (hierarchical\_sales/D); all $97$ configs map to the supercritical regime of the surrogate. The Gavish-Donoho optimal hard threshold $4/\sqrt{3} \approx 2.31$~\citep{gavish2014optimal} is exceeded by a wide margin on all first spikes, and BIC period selection keeps the mapped first spike well above the surrogate's noise floor.

\subsection{Admissibility note}
\label{app:bbp:admissibility}

\begin{remark}[Admissibility on the simplex, motivation]
\label{thm:shape_optimal}
Under a multinomial idealization, the pooled sample proportion is the UMVUE of $\mathbf{S}$~\citep{lehmann1950completeness} and admissible under normalized squared-error on $\Delta^{P-1}$~\citep{olkin1979admissible}; James-Stein shrinkage does not uniformly dominate on the simplex. FLAIR's equal-weight $K{=}2$ average matches the pooled MLE to first order under LSR1 slow-amplitude, so any smoothing/shrinkage variant would have to improve on a constant-factor approximation of an admissible estimator (the empirical 8-variant sweep above shows none does at $95\%$ confidence under Holm-$8$ correction).
\end{remark}

\section{Ablations and Robustness}
\label{app:paired_bootstrap}
\label{app:ablation_table}

\subsection{Full 97-configuration ablation}

\begin{table}[h]
\centering
\caption{Ablation on all $97$ GIFT-Eval configurations (geometric mean; lower is better). Three load-bearing components are ablated here; the Gavish-Donoho Frobenius shrinkage step~\citep{gavish2014optimal} is omitted because the BIC step keeps FLAIR in the high-SNR regime where the shrinkage factor is $\approx 1$ (aggregate $\Delta < 0.03\%$, per-config $|\Delta\,\mathrm{MASE}| \leq 0.003$).}
\label{tab:ablation}
\small
\begin{tabular}{lccc}
\toprule
Variant & relMASE & relCRPS & $\Delta$ relMASE \\
\midrule
FLAIR (full) & \textbf{0.838} & \textbf{0.587} & --- \\
w/o $P{=}1$ null model & 0.837 & 0.588 & $-0.2\%$ \\
w/o location shift & 0.847 & 0.591 & $+1.0\%$ \\
w/o prior-centered Ridge & 0.855 & 0.600 & $+2.0\%$ \\
\bottomrule
\end{tabular}
\end{table}

A separate exploration of context conditioning (day-of-week Shape for hourly data via Dirichlet-Multinomial smoothing) on the $97$ configurations was marginally harmful at aggregate; we removed it from the shipped code, so it does not appear in Table~\ref{tab:ablation}.

\subsection{Robustness Analyses}

Below: a paired bootstrap on FLAIR versus PatchTST aggregate metrics and one internal-constant ablation (Ridge GCV softmax temperature).
The internal-constant ablation uses a 16-configuration subset of GIFT-Eval spanning six frequency bands (5-minute to monthly), $n_{\text{samples}}{=}200$.

\paragraph{FLAIR vs.\ PatchTST paired bootstrap.}
The aggregate gap on GIFT-Eval (FLAIR $0.838$ vs PatchTST $0.849$ relMASE; $0.587$ vs $0.587$ relCRPS) is small.
Both methods are deterministic per series, so the only source of aggregate variability is which configurations the geometric mean is computed over.
To match the grouped bootstrap we use for the $8$-variant shape sweep (Figure~\ref{fig:frozen}), we resample at the $28$-dataset-family level with $B = 10{,}000$ replications, recompute the $\log$-ratio geometric means, and report the ratio $\mathrm{GM}(\text{FLAIR})/\mathrm{GM}(\text{PatchTST})$ with its $95\%$ percentile interval: relMASE ratio $0.988$, CI $[0.948,\ 1.037]$, $\Pr[\text{FLAIR better}] = 71.3\%$; relCRPS ratio $1.000$, CI $[0.951,\ 1.049]$, $\Pr[\text{FLAIR better}] = 49.6\%$.
The config-level (unclustered) bootstrap gives essentially the same intervals ($[0.949,\ 1.030]$ for relMASE, $[0.957,\ 1.045]$ for relCRPS); cluster correlation within dataset families is small in this sample.
Both CIs include $1.0$, so the aggregate difference is within sampling noise; we claim a tie rather than equivalence (no pre-specified margin was registered).
FLAIR reaches this band with no user-facing hyperparameters, no per-dataset tuning, and no GPU.

\paragraph{Ridge GCV softmax temperature.}
FLAIR averages $25$ log-spaced $\alpha$ solutions with softmax weights $w_k \propto \exp\!\big({-}(\text{GCV}_k - \text{GCV}_{\min}) / \text{GCV}_{\min}\big)$ (Eq.~\ref{eq:ridge_sa}).
The temperature $\text{GCV}_{\min}$ is scale-dependent on innovation variance.
We compare three alternatives: \emph{argmin} (hard selection, temperature $\to 0$), the current \emph{gcv\_min} softmax, and \emph{median\_gcv} (temperature $=\mathrm{median}(\text{GCV}_k)$, scale-invariant).
Soft-averaging beats hard selection by 1.1\% relMASE / 2.0\% relCRPS; the 25-alpha averaging earns its keep.
Swapping the scale-dependent \textsc{gcv\_min} for the scale-invariant \textsc{median\_gcv} moves the aggregate by $< 0.2\%$: the scale dependence the softmax inherits from innovation variance does not translate into a meaningful sensitivity.

\begin{table}[h]
\centering
\caption{Ridge GCV softmax temperature ablation on 16 GIFT-Eval configurations. Geometric-mean relMASE / relCRPS against Seasonal Naive.}
\label{tab:internal_ablations}
\small
\begin{tabular}{lcc}
\toprule
Temperature & relMASE & relCRPS \\
\midrule
\textsc{gcv\_min} (current)       & \textbf{0.861} & \textbf{0.599} \\
\textsc{median\_gcv} (scale-free) & 0.862          & 0.600          \\
\textsc{argmin} (hard)            & 0.871          & 0.611          \\
\bottomrule
\end{tabular}
\end{table}

\section{Shape Window $K$ Sweep}
\label{app:ksweep}

The internal constant $\texttt{\_SHAPE\_K}{=}2$ was set from a 5-probe-dataset sensitivity check with $K \in [2,50]$. To verify this extends to the full benchmark, Table~\ref{tab:ksweep} sweeps $K \in \{1, 2, 4, 5, 10, 50, 100, n_c\}$ on all $97$ GIFT-Eval configurations under the same evaluation protocol as Table~\ref{tab:gift_eval} (same $N_\mathrm{samples}{=}200$, same metric definitions, same per-config aggregation). $K{=}2$ ties $K{=}5,\, K{=}50,\, K{=}100$ inside the admissible band (CIs cross zero) and is significantly better than $K{=}1,\, K{=}4,\, K{=}10,\, K{=}n_c$ (CIs above zero). The paragraph thus behaves like a robustness check: across $K$ values spanning two orders of magnitude, FLAIR's aggregate is stable within $1\%$ and $K{=}2$ sits on the plateau.

\begin{table}[h]
\centering
\caption{K-sweep on all $97$ GIFT-Eval configurations (geometric mean; lower is better). $\Delta$ is log-ratio delta vs. the $K{=}2$ baseline in percent, with grouped paired-bootstrap ($B{=}10{,}000$, 28 dataset families) $95\%$ CI.}
\label{tab:ksweep}
\small
\begin{tabular}{rccc}
\toprule
$K$ & relMASE & $\Delta$ vs $K{=}2$ (\%) & 95\% CI (\%) \\
\midrule
$1$     & 0.846 & $+0.99$ & $[+0.90,\,+1.10]$ \\
\textbf{$2$ (FLAIR)} & \textbf{0.838} & $0.00$ & --- \\
$4$     & 0.845 & $+0.86$ & $[+0.77,\,+0.91]$ \\
$5$     & 0.838 & $+0.02$ & $[-0.82,\,+0.72]$ \\
$10$    & 0.846 & $+0.92$ & $[+0.79,\,+1.02]$ \\
$50$    & 0.840 & $+0.27$ & $[-0.59,\,+1.02]$ \\
$100$   & 0.840 & $+0.23$ & $[-0.59,\,+0.96]$ \\
$n_c$   & 0.847 & $+1.05$ & $[+0.77,\,+1.43]$ \\
\bottomrule
\end{tabular}
\end{table}

\section{Graceful-Degeneration Fallback and Probabilistic Calibration}
\label{app:cascade_proof}

\subsection{Empirical coverage of FLAIR's prediction intervals}
\label{app:pit}

\begin{table}[h]
\centering
\caption{Empirical coverage rates of FLAIR's prediction intervals on $7$ Chronos benchmark datasets. Nominal levels in parentheses. Hospital, tourism, and nn5 are well-calibrated at the $80\%$ and $90\%$ levels. Traffic is systematically underdispersed (the multiplicative model underestimates within-day heterogeneity). Exchange rate and fred\_md are overdispersed. The miscalibration direction varies by dataset, consistent with the structural ceiling of the unimodal $y = L \times S \times (1+R)$ predictive distribution; no single post-hoc correction helps both directions.}
\label{tab:coverage}
\small
\begin{tabular}{lcccc}
\toprule
Dataset & Cov(50\%) & Cov(80\%) & Cov(90\%) & Cov(95\%) \\
\midrule
hospital & 55.8\% & 82.2\% & 89.1\% & 92.8\% \\
tourism\_monthly & 55.0\% & 82.9\% & 91.1\% & 95.0\% \\
m4\_quarterly & 48.5\% & 75.4\% & 84.8\% & 90.0\% \\
nn5 & 54.4\% & 82.9\% & 91.3\% & 94.9\% \\
exchange\_rate & 52.5\% & 96.2\% & 99.6\% & 99.6\% \\
traffic & 35.4\% & 64.7\% & 78.2\% & 86.1\% \\
fred\_md & 69.8\% & 97.0\% & 98.6\% & 99.1\% \\
\bottomrule
\end{tabular}
\end{table}

\subsection{Rate derivations for the fallback cascade (Section~\ref{sec:cascade})}

The three-branch cascade is a direct consequence of four design choices in the pipeline (Section~\ref{sec:method}), each of which is rate-controlled.

\paragraph{Branch 1 (rank-1 closed form).}
The upper bound is Theorem~\ref{thm:joint_rank1} (joint upper bound for LSR1 with $K$-sample proportion Shape), with the FLAIR-specific Ridge-bias penalty $O(\|R\|_\infty^2)$ from Appendix~\ref{app:ridge_bound}, Corollary~\ref{cor:competitive}.
The branch triggers when $\hat{P} \geq 2$, $n_c \geq 3$, and the Ridge stability guard $n_{\text{train}} \geq 2p$ holds (the DoF check inside Eq.~\ref{eq:bic}).
Within this branch, the Ridge objective in Eq.~\ref{eq:prior_centered} with prior centre $\boldsymbol{\beta}^* = (0, 0, 1, 0)^\top$ converges to $\boldsymbol{\beta} = \boldsymbol{\beta}^*$ as $\alpha \to \infty$: the deviation $\boldsymbol{\delta}$ is shrunk to zero, so $\hat{L}(h) = L_{n_c}$ and reconstruction via Eq.~\ref{eq:core} yields $\hat{y}_h = L_{n_c} \cdot \hat{S}_{h \bmod P}$, which is Seasonal Naive at the $K$-period Shape level.
GCV soft-averaging (Eq.~\ref{eq:ridge_sa}) keeps the path inside the $[10^{-4}, 10^{4}]$ $\alpha$-grid, so the Seasonal Naive prediction sits at the right endpoint of the regularization path; the Level Ridge can only relax the forecast towards it, never past it. This is a structural property of the path, not a loss guarantee: failures in Table~\ref{tab:failure_top} (e.g., $n_c{<}P$ short-cycle configs) still produce relMASE${>}1$.

\paragraph{Branch 2 (plain Ridge on raw series).}
Whenever Branch 1 fails (BIC picks $\hat P = 1$, $\hat P \geq 2$ but $n_c < 3$, or the DoF guard fires), the pipeline sets $P = 1$, which turns the reshape into a row vector: Level becomes $y$ itself and Shape collapses to the scalar $1$.
The Level Ridge then runs on $y$ with the same features (intercept, linear drift, AR(1), and secondary-period lag when applicable); under exact linearity the bias is zero and the variance rate is $O(\sigma^2/n)$.
All three triggers collapse to the same forecast behaviour: ``no periodicity detected'' is an empty Shape multiplying a Ridge forecast on the raw series.

\paragraph{Branch 3 (last-value Gaussian fallback).}
When $n < 3$ after setting $P{=}1$ (\texttt{\_MIN\_COMPLETE} in the implementation), Ridge itself is ill-posed: the pipeline returns $\hat{y}_h = y_n$ with Gaussian noise scaled to the last $K$ lagged differences.
This is the only branch that predicts a constant last observation; it fires only on series too short for any ridge fit and does not appear on the GIFT-Eval benchmark.

\paragraph{Integer-snap consistency.}
Integer snapping ($\tilde Y \leftarrow \mathrm{round}(\tilde Y)$ on integer-valued training data) is applied identically across all three branches, so integer-consistency is preserved at the cascade boundary regardless of which branch fires.

\paragraph{Deterministic transitions.}
Every trigger ($\hat{P}$, $n_c$, $n_{\text{train}} \geq 2p$, integer-valued $y$) is a function of the training window alone.
No trigger uses test-horizon or test-data information, so the cascade has no test-time data leakage risk.

\section{Classical Seasonal Baselines: STL and TBATS}
\label{app:stl}

\subsection{STL with Ridge backend}

A natural objection: STL~\citep{cleveland1990stl} also separates trend from seasonality, so is FLAIR merely STL with a Ridge backend?
STL decomposes additively ($y = \text{trend} + \text{seasonal} + \text{residual}$); FLAIR decomposes multiplicatively ($y = \text{Level} \times \text{Shape}$).
Give STL the same Ridge regression (the same Eq.~\ref{eq:ridge_level} features and the same GCV softmax averaging), the same Box-Cox transform, and a \texttt{statsmodels} STL call with seasonal window $=P$, robust weights on, one outer-loop iteration (the \texttt{statsmodels.tsa.seasonal.STL} defaults except \texttt{seasonal\_deg=1}).
The result on GIFT-Eval (94 of 97 configurations; 3 configs fail the \texttt{statsmodels} STL length requirement): relMASE\,=\,1.464, worse than Seasonal Naive.
FLAIR achieves 0.838.
The multiplication matters: when you subtract the seasonal component instead of dividing by it, the deseasonalized residual inherits the amplitude variation that FLAIR's Level naturally absorbs.
We do not rule out that a tuned STL (longer seasonal window, different robust weighting) narrows this gap, but a direct-swap comparison with FLAIR's own Ridge hyperparameters is the relevant one.

\subsection{AutoTBATS}
\label{app:mstl_tbats}

AutoTBATS~\citep{delivera2011tbats} is not in the GIFT-Eval paper's baseline set~\citep{gift_eval_2024}; we run it via \texttt{statsforecast}~\citep{garza2022statsforecast} to close the gap with the closest structural state-space competitor.

\paragraph{AutoTBATS (implementation note).}
\texttt{statsforecast.models.AutoTBATS} auto-disables the Box-Cox transform when any $y \leq 0$ (GIFT-Eval has several zero-heavy configurations, e.g., \texttt{car\_parts}, \texttt{hospital}).
With Box-Cox disabled, the internal log-likelihood computation $\mathrm{ll} = n \log \sum_t e_t^2$ overflows \texttt{float64} for series where the residual sum of squares grows large; the optimizer restarts on NaN and loops.
This is a \emph{known numerical-stability gap in the \texttt{statsforecast} port} relative to the original R \texttt{forecast::tbats} implementation~\citep{delivera2011tbats}, which handles non-positive series by auto-shifting before the Box-Cox transform and evaluates the likelihood on the transformed scale (avoiding the $e_t^2$ accumulation).
We are filing a minimal repro as a GitHub issue against \texttt{statsforecast} in parallel with this paper.
On GIFT-Eval we therefore report AutoTBATS only where the Box-Cox-disabled path completes (e.g., strictly-positive electricity, weather, solar) and exclude the zero-heavy configurations where it overflows.
The method itself is not faulted: TBATS with Box-Cox and a log1p shift (or with $y' = y + \epsilon$ preprocessing) has been used in production on intermittent data for over a decade; the issue is that the \texttt{statsforecast} default path bypasses this shift on our benchmark's data distribution.
On the Chronos benchmark (Table~\ref{tab:chronos}), which is strictly-positive, AutoTBATS converges and reaches $0.744$ MASE.


\section{Finite-Sample MSE of Global Prior-Centered Ridge}
\label{app:ridge_bound}

Theorem~\ref{thm:rank1_sufficiency} and the LSR1 framework (Section~\ref{sec:theory}) establish the minimax rate $O(n_c^{-4/5})$ for local linear estimators of $L(1)$.
FLAIR uses a different estimator: global prior-centered Ridge with features $[1,\; t/n_c,\; -L^{\mathrm{innov}}_{i-1}]$ and uniform weight on all $n_c$ training points.
A global linear fit has approximation bias that does not vanish with $n_c$, so FLAIR cannot achieve the nonparametric minimax rate in general.
This appendix derives the exact finite-sample MSE, identifies when FLAIR's global Ridge is competitive with the minimax-optimal local linear estimator, and characterizes the regime in terms of the curvature $|L''|$, noise level $\sigma$, and sample size $n_c$.

\subsection{Setup}

Adopt the LSR1 model with Level $L \in \mathrm{H\ddot{o}lder}(2)$ on $[0,1]$.
FLAIR's Level aggregation is $\tilde{L}_i = \sum_j M_{j,i}$, so with $\|S\|_1 = 1$ (simplex constraint) we have
\[
    \tilde{L}_i = L(i/n_c) + \eta_i, \quad i = 1, \ldots, n_c, \qquad \eta_i = \sum_{j=1}^{P} \varepsilon_{j,i},
\]
i.i.d.\ with variance $\sigma_e^2 := P\sigma^2$ (Appendix~\ref{app:proof_rank1_sufficiency}).
The forecast at phase $j$ is $\hat{y}_j = \hat{L}(1) S_j$, which scales the Level variance by $S_j^2$; summed over phases the forecast MSE carries a factor $\|S\|_2^2 \leq 1$, so all bounds in this section on $\mathrm{MSE}[\hat{L}(1)]$ translate to forecast bounds via this factor.

After the NLinear centering $L^{\mathrm{innov}}_i = L^{(\lambda)}_i - L^{(\lambda)}_{n_c}$ and the prior-centered reparametrization (Eq.~\ref{eq:diff_target}), the effective regression is
\[
    \Delta L_i = L(i/n_c) - L((i{-}1)/n_c) = \delta_0 + \delta_1(i/n_c) + \delta_2[L((i{-}1)/n_c) - L(1)] + \epsilon_i,
\]
where $\delta_0, \delta_1, \delta_2$ are all small under the H\"{o}lder(2) assumption, and $\epsilon_i$ captures both noise and model misspecification.

To isolate the approximation bias, expand $L$ around $u = 1$.
Let $L(u) = a + b(u-1) + \frac{c}{2}(u-1)^2 + R(u)$ where $a = L(1)$, $b = L'(1)$, $c = L''(1)$, and $|R(u)| \leq C_\beta |u-1|^2 \omega(|u-1|)$ for a modulus of continuity $\omega$ arising from the H\"{o}lder condition on $L''$.

\subsection{Global Ridge: Boundary Bias}

\begin{proposition}[Boundary bias of global Ridge]
\label{prop:ridge_bias}
Let $L(u) = a + b(u-1) + \frac{c}{2}(u-1)^2 + R(u)$ with $|R(u)| = O(|u-1|^{2+\epsilon})$ for some $\epsilon > 0$.
The prior-centered Ridge estimator with features $[1,\; t/n_c,\; -L^{\mathrm{innov}}_{i-1}]$ and regularization $\alpha$ satisfies, at the boundary $u = 1$:
\begin{equation}
    \mathrm{Bias}[\hat{L}(1)] = \mathrm{Bias}_R \cdot B(\alpha, n_c),
    \label{eq:ridge_bias}
\end{equation}
where $\mathrm{Bias}_R = O(\|R\|_\infty)$ is the component of the H\"{o}lder-remainder $R$ not captured by the affine + autoregressive feature span (vanishes for an exact quadratic $L$), and $B(\alpha, n_c) \in [0, 1]$ is the shrinkage-path factor with $B(\infty, n_c) = 0$ and $B(0, n_c) \leq 1$, so $\mathrm{Bias}[\hat{L}(1)] = 0$ for exact quadratic $L$ at every $\alpha$, and $|\mathrm{Bias}[\hat{L}(1)]| \leq \mathrm{Bias}_R$ in general.
\end{proposition}

\begin{proof}
Under the quadratic-plus-remainder model, the differenced target is
\[
    \Delta L_i = \frac{b}{n_c} + \frac{c}{n_c}\Big(\frac{i}{n_c} - 1 - \frac{1}{2n_c}\Big) + \Delta R_i,
\]
affine in $i/n_c$ up to $\Delta R_i = R(i/n_c) - R((i{-}1)/n_c)$.
The feature matrix spans $\{1, i/n_c\}$ plus the autoregressive direction $-L^{\mathrm{innov}}_{i-1}$.
At $\alpha = 0$, OLS on an exact quadratic ($R \equiv 0$) recovers the true coefficients; direct computation gives $\hat{\delta}_0 + \hat{\delta}_1 = b/n_c - c/(2n_c^2) = \Delta L_{n_c+1}$, hence $\mathrm{Bias} = 0$.
At $\alpha \to \infty$, $\hat{\boldsymbol{\delta}} \to 0$ and $\hat{L}(1) \to L_{n_c}$; since $\mathbb{E}[L_{n_c}] = L(1)$, $\mathrm{Bias} = 0$ at this limit as well.
For general $L \in$ H\"{o}lder(2), only the residual component $\Delta R_i$ lies outside the feature span.
Write the OLS coefficient bias as a linear projection onto the orthogonal complement of the feature span: at $\alpha = 0$, $\mathrm{Bias}_{\alpha=0} \leq \mathrm{Bias}_R := O(\|R\|_\infty)$.
For intermediate $\alpha$, Ridge shrinks $\hat{\boldsymbol{\delta}}$ toward zero, and the forecast bias inherits this shrinkage factor $B(\alpha, n_c) \in [0, 1]$, giving $|\mathrm{Bias}| \leq \mathrm{Bias}_R \cdot B(\alpha, n_c) \leq \mathrm{Bias}_R$.
For an exact quadratic ($\mathrm{Bias}_R = 0$), the forecast is unbiased at every $\alpha$.
\end{proof}

\subsection{Exact MSE Decomposition}

\begin{theorem}[Finite-sample MSE of global Ridge at the boundary]
\label{thm:ridge_mse}
Under the setup above with $L(u) = a + b(u-1) + (c/2)(u-1)^2 + R(u)$, $|R(u)| = O(|u-1|^{2+\epsilon})$ for some $\epsilon > 0$, the MSE of FLAIR's prior-centered Ridge estimator at $u = 1$ satisfies:
\begin{equation}
    \mathrm{MSE}[\hat{L}(1)] = \underbrace{\mathrm{Bias}_R^2 \cdot B(\alpha, n_c)^2}_{\mathrm{bias}^2} + \underbrace{\sigma_e^2 \cdot V(\alpha, n_c)}_{\mathrm{variance}},
    \label{eq:ridge_mse_exact}
\end{equation}
where:
\begin{itemize}
    \item $\mathrm{Bias}_R$ is the H\"{o}lder-remainder bias from Proposition~\ref{prop:ridge_bias}; it is zero for an exact quadratic $L$ and $O(|R|_\infty)$ in general, so the bias$^2$ contribution is $O(n_c^{-(2+\epsilon)} \cdot B^2)$ at the boundary;
    \item $B(\alpha, n_c)$ is the shrinkage factor from Proposition~\ref{prop:ridge_bias}, with $B(0, n_c) = B(\infty, n_c) = 0$ and $|B| \leq 1$ elsewhere;
    \item $V(\alpha, n_c) = \mathbf{x}_*^\top (\mathbf{X}^\top \mathbf{X} + \alpha \mathbf{I})^{-1} \mathbf{X}^\top \mathbf{X} (\mathbf{X}^\top \mathbf{X} + \alpha \mathbf{I})^{-1} \mathbf{x}_*$ is the prediction variance at the boundary, with $\mathbf{x}_* = (1, 1, 0)^\top$ being the feature vector for the forecast point.
\end{itemize}
For the GCV-optimal $\alpha$:
\begin{equation}
    V(\alpha^*, n_c) = O(1/n_c).
    \label{eq:ridge_var}
\end{equation}
\end{theorem}

\begin{proof}
The bias-variance decomposition is standard for fixed-design regression~\citep{hastie2009elements}.
The bias term follows from Proposition~\ref{prop:ridge_bias}: the only systematic error is the unmodeled curvature.
The variance $V(\alpha, n_c) = \mathbf{x}_*^\top (\mathbf{X}^\top \mathbf{X} + \alpha \mathbf{I})^{-1} \mathbf{X}^\top \mathbf{X} (\mathbf{X}^\top \mathbf{X} + \alpha \mathbf{I})^{-1} \mathbf{x}_*$ is the standard Ridge prediction variance at $\mathbf{x}_*$.

For the variance rate, assume the usual design regularity $\mathbf{X}^\top \mathbf{X} / n_c \to \boldsymbol{\Sigma}$ with $\boldsymbol{\Sigma} \succ 0$ (the empirical Gram matrix is bounded and non-degenerate).
Let $d_j$ be the singular values of $\mathbf{X}$.
Under the regularity, $d_j^2 = \Theta(n_c)$ for $j = 1, \ldots, p$, so
\[
    V(\alpha, n_c) \leq \|\mathbf{x}_*\|_2^2 \cdot \max_j \frac{d_j^2}{(d_j^2 + \alpha)^2} \leq \frac{\|\mathbf{x}_*\|_2^2}{\min_j d_j^2} = O(1/n_c).
\]
For $\mathbf{x}_* = (1, 1, 0)^\top$, $\|\mathbf{x}_*\|_2^2 = 2$, and the bound holds for any $\alpha \geq 0$.
Since $p = 3$ is fixed, $V(\alpha^*, n_c) = O(1/n_c)$.
\end{proof}

\subsection{Comparison with Local Linear Minimax MSE}

The minimax-optimal local linear estimator for H\"{o}lder(2) functions at the boundary achieves~\citep{fan1996local, tsybakov2009nonparametric}:
\begin{equation}
    \mathrm{MSE}_{\mathrm{loc}}^*[L(1)] = C_{\mathrm{opt}} \cdot \sigma_e^{2} \cdot n_c^{-4/5},
    \label{eq:local_linear_minimax}
\end{equation}
where $C_{\mathrm{opt}}$ depends on the kernel and the H\"{o}lder constant but not on the specific function $L$.
The bias of the local linear estimator at the boundary is $O(c \cdot h^2)$ with optimal bandwidth $h^* \propto n_c^{-1/5}$, giving bias$^2 = O(c^2 n_c^{-4/5})$ and variance $= O(\sigma_e^2 n_c^{-4/5})$.

\begin{corollary}[Global Ridge competitive regime]
\label{cor:competitive}
FLAIR's global Ridge achieves MSE no worse than the minimax-optimal local linear estimator when:
\begin{equation}
    \mathrm{Bias}_R^2 \cdot B(\alpha, n_c)^2 + \frac{\sigma_e^2}{n_c} \;\leq\; C_{\mathrm{opt}} \cdot \sigma_e^2 \cdot n_c^{-4/5}.
    \label{eq:competitive_condition}
\end{equation}
Since $B \leq 1$ and $\mathrm{Bias}_R = O(\|R\|_\infty) = O(n_c^{-(2+\epsilon)})$ at the boundary, the bias contribution decays as $n_c^{-(2+\epsilon)}$; the binding condition is whether $\sigma_e^2/n_c \leq C_{\mathrm{opt}}\sigma_e^2 n_c^{-4/5}$, i.e., whether $n_c^{1/5} \leq C_{\mathrm{opt}}$.
For moderate $n_c$ ($n_c^{1/5} \leq C_{\mathrm{opt}}$), Ridge's $1/n_c$ variance rate is competitive with the local linear $n_c^{-4/5}$ rate, and the H\"{o}lder-remainder bias is asymptotically negligible.
For an exact quadratic $L$, $\mathrm{Bias}_R = 0$ and Ridge is unconditionally competitive: bias vanishes and variance is $O(\sigma_e^2/n_c)$, strictly better than the local linear minimax.
\end{corollary}

\paragraph{Interpretation.}
Corollary~\ref{cor:competitive} says the $\mathrm{Bias}_R^2$ contribution (H\"{o}lder-remainder only) decays at rate $n_c^{-(2+\epsilon)}$, strictly faster than the $n_c^{-4/5}$ minimax rate driven by the second-derivative coefficient.
The binding term is the Ridge variance $\sigma_e^2/n_c$, which is competitive with the minimax $\sigma_e^2 n_c^{-4/5}$ over a wide range of $n_c$ (for typical GIFT-Eval $n_c \in [10, 100]$, the two variance rates are within a factor of two).
The practical reading: whenever $L$ is smoother than strictly H\"{o}lder-2 (i.e., the remainder $R$ is small), global Ridge is competitive with local linear regardless of the quadratic curvature coefficient $|L''(1)|$.
\subsection{Why Global Ridge Wins in Practice: The NLinear Anchor}
\label{app:nlinear_anchor}

The analysis above treats the H\"{o}lder-remainder bias as the dominant error.
In practice, the NLinear centering $L^{\mathrm{innov}}_i = L^{(\lambda)}_i - L^{(\lambda)}_{n_c}$ changes the picture.
The forecast at the boundary uses
\[
    \hat{L}(1) = L(1) + \hat{\boldsymbol{\delta}}^\top \mathbf{x}_* = L(1) + \hat{\delta}_0 + \hat{\delta}_1 + 0,
\]
where the last feature is $-L^{\mathrm{innov}}_{n_c} = 0$ by construction.
When $\alpha$ is large (strong regularization), $\hat{\boldsymbol{\delta}} \to 0$ and $\hat{L}(1) \to L(1)$ with zero bias and variance equal to $\mathrm{Var}(\tilde{L}_{n_c}) = \sigma_e^2$, i.e., the NLinear estimator.

The GCV-selected $\alpha$ interpolates between two extremes:
\begin{center}
\renewcommand{\arraystretch}{1.2}
\begin{tabular}{lcc}
\toprule
Estimator & Bias$^2$ & Variance \\
\midrule
NLinear ($\alpha \to \infty$) & $0$ & $\sigma_e^2$ \\
OLS ($\alpha = 0$, exact quadratic $L$) & $0$ & $O(\sigma_e^2/n_c)$ \\
OLS ($\alpha = 0$, H\"{o}lder-2 $L$) & $O(\|R\|_\infty^2)$ & $O(\sigma_e^2/n_c)$ \\
Optimal Ridge (H\"{o}lder-2 $L$) & $O(\|R\|_\infty^2)$ & $O(\sigma_e^2/n_c)$ \\
Local linear (minimax) & $O(c^2 n_c^{-4/5})$ & $O(\sigma_e^2 n_c^{-4/5})$ \\
\bottomrule
\end{tabular}
\end{center}

The NLinear anchor provides a floor: FLAIR's MSE is at most $\sigma_e^2$ (the single-observation variance), regardless of curvature.
This floor is already $O(1/P)$ of the Seasonal Naive MSE (Theorem~\ref{thm:rank1_sufficiency}), so the rank-1 reduction is preserved even when Ridge adds no improvement.
When Ridge does help (low curvature, sufficient data), it reduces variance from $\sigma_e^2$ to $O(\sigma_e^2/n_c)$, an additional $n_c$-fold gain.

FLAIR's global Ridge is competitive with the local-linear estimator when $c^2 \lesssim \sigma_e^2 n_c^{-4/5}$ (Level curvature dominated by noise); for large $n_c$ or low noise, local linear is superior because its bias vanishes. FLAIR pays an $O(c^2)$ approximation bias but the NLinear anchor guarantees bounded MSE and the $O(1/n_c)$ Ridge variance rate is faster than the nonparametric $n_c^{-4/5}$ for moderate $n_c$.

\begin{remark}[Practical implications]
For $n_c \leq 100$ (typical in GIFT-Eval), the variance advantage of global Ridge ($1/n_c$ vs $n_c^{-4/5}$) exceeds the bias penalty unless $|L''|$ is large.
At $n_c = 100$: $1/n_c = 0.01$ while $n_c^{-4/5} = 0.016$, so the variance terms are comparable.
The real advantage of global Ridge is in the constants: 3 parameters vs an effective bandwidth worth $n_c^{4/5} \approx 63$ parameters of information from local fitting.
The NLinear anchor prevents catastrophic failure in high-curvature regimes where local methods can also struggle (boundary effects with short windows).
\end{remark}

\section{Full Chronos Zero-Shot Benchmark Table and Leakage-Flagged Entries}
\label{app:chronos_table}

\paragraph{Entries excluded from Table~\ref{tab:gift_eval} due to self-declared leakage.}
Two GIFT-Eval leaderboard entries self-declare ``testdata\_leakage: Yes'' in their \texttt{config.json}~\citep{gift_eval_2024} and are excluded from Table~\ref{tab:gift_eval}: TimesFM v1 (aggregate relMASE $1.077$, relCRPS $0.680$; $200$M params; Google Research) and Lag-Llama (relMASE $1.228$, relCRPS $0.880$; $200$M params; Morgan Stanley \& ServiceNow). We report the numbers here for reference but treat them as uncomparable because the protocol GIFT-Eval enforces (training and evaluation sets are disjoint) is self-reported as violated. The clean successor TimesFM-2.5 (no leakage flag) is included in Table~\ref{tab:gift_eval}.

\begin{table}[h]
\centering
\caption{Chronos zero-shot benchmark (25 datasets). Agg.\ Relative Score = geometric mean of per-dataset ratios to Seasonal Naive. Lower is better. The upper block fits each test series from its own history; the lower block is zero-shot, pre-trained on a separate corpus.}
\label{tab:chronos}
\small
\begin{tabular}{llccc}
\toprule
Model & Fitting & Params & Agg.\ Rel.\ MASE & Agg.\ Rel.\ WQL \\
\midrule
\multicolumn{5}{l}{\emph{Per-series fitting (uses each series' history)}} \\
\textbf{FLAIR} & per-series & $P{+}p{+}1$ & \textbf{0.678} & \textbf{0.716} \\
AutoTBATS & per-series & varies & 0.744 & 0.860 \\
AutoARIMA & per-series & varies & 0.865 & 0.742 \\
\midrule
\multicolumn{5}{l}{\emph{Zero-shot (pre-trained, no adaptation to test series)}} \\
Chronos-Bolt-Base & zero-shot & 205M & 0.791 & \textbf{0.624} \\
Moirai-Base & zero-shot & 311M & 0.812 & 0.637 \\
Chronos-T5-Large & zero-shot & 710M & 0.821 & 0.650 \\
Chronos-T5-Small & zero-shot & 46M & 0.830 & 0.665 \\
TimesFM & zero-shot & 200M & 0.879 & 0.711 \\
\bottomrule
\end{tabular}
\end{table}

\section{Failure Regimes and Routing Rules}
\label{app:failure_diag}
\label{app:guidelines}

\begin{table}[h]
\centering
\caption{Failure-mode diagnostic on GIFT-Eval (relMASE $>1$, top $10$ by severity). Each row identifies the observable signature from the training window and the baseline to use instead. $r_1$ is centered rank-1 energy, $n_c$ is median complete periods, $P$ is inferred period; ``--'' for $r_1$ marks weekly/short series where the centered statistic was not computed.}
\label{tab:failure_top}
\small
\begin{tabular}{lrrrrp{3.0cm}p{2.5cm}}
\toprule
Config & relMASE & $r_1$ & $n_c$ & $P$ & Observable signature & Use instead \\
\midrule
m4\_weekly/W/short     & 2.31 & 0.71 & 25  & 52  & $n_c/P < 1$              & Seasonal Naive \\
m4\_daily/D/short      & 1.60 & 0.63 & 368 & 7   & $r_1 < 0.7$              & STL / ETS \\
m4\_hourly/H/short     & 1.52 & 0.97 & 42  & 24  & small $n_c$              & Seasonal Naive \\
ett2/W/short           & 1.45 & --   & --  & 52  & $n_c/P < 1$              & Seasonal Naive \\
ett2/H/long            & 1.32 & 0.52 & 500 & 24  & $r_1 < 0.7$              & STL / ETS \\
solar/W/short          & 1.28 & --   & --  & 52  & small $n_c$              & Seasonal Naive \\
ett2/D/short           & 1.22 & 0.60 & 103 & 7   & $r_1 < 0.7$              & STL / ETS \\
ett2/15T/long          & 1.06 & 0.47 & 500 & 96  & $r_1 < 0.7$              & STL / ETS \\
jena\_weather/H/long   & 1.05 & 0.69 & 366 & 24  & $r_1 < 0.7$              & STL / ETS \\
solar/10T/long         & 1.04 & 0.93 & 365 & 144 & structural zeros         & Croston / count model \\
\bottomrule
\end{tabular}
\end{table}

Low $r_1$ alone does not predict failure: \texttt{sz\_taxi} has $r_1 = 0.22$ but relMASE $0.75$, because its large-$P$ structure still admits a stable Level forecast.

\paragraph{Regimes the method does not target.}
Two practitioner regimes fall outside what FLAIR is designed for. \emph{Intermittent demand} (Croston/TSB territory~\citep{croston1972forecasting}): when a large fraction of observations are exact zeros, the multiplicative Shape $\mathbf{S} \in \Delta^{P-1}$ constructed from ratios $M_{j,k}/L_k$ is ill-conditioned, so Croston/TSB or a count-distribution forecaster is the right choice on \texttt{car\_parts}, \texttt{hospital}. \emph{Cross-series hierarchies} (M5, tourism, hierarchical\_sales): FLAIR treats each series independently; a MinT/ERM~\citep{wickramasuriya2019mint} post-processor would be the correct complement.

\paragraph{Decision rule.}
The diagnostics above reduce to a four-step rule: (i) if MDL selects $\hat P = 1$, no periodic structure is detectable, so FLAIR degenerates to plain Ridge and a foundation model or AutoARIMA is preferable; (ii) if $n_c < 5$ complete cycles, the Level series is too short, so use Seasonal Naive or a foundation model; (iii) if centered $r_1 < 0.5$, the rank-1 structure is weak, and a foundation model that exploits cross-series patterns may win; (iv) otherwise, FLAIR matches or exceeds statistical baselines and approaches foundation models at a fraction of the compute.

\end{document}